\documentclass[a4paper,11pt]{article}

\usepackage[margin=1.0in]{geometry}

\usepackage{amsmath}
\usepackage{algorithmic}
\usepackage{algorithm}
\usepackage{amsthm}
\usepackage{amsfonts}
\usepackage{comment}
\usepackage{graphicx}
\usepackage{color}
\usepackage{subcaption}
\usepackage{url}

\newtheorem{theorem} {Theorem}
\newtheorem{lemma} {Lemma}

\def\x{{\mathbf{x}}}

\def\v{{\mathbf{v}}}
\def\z{{\mathbf{z}}}
\def\w{{\mathbf{w}}}
\def\y{{\mathbf{y}}}

\def\X{{\mathbf{X}}}
\def\Y{{\mathbf{Y}}}
\def\A{{\mathbf{A}}}
\def\M{{\mathbf{M}}}
\def\I{{\mathbf{I}}}
\def\B{{\mathbf{B}}}

\def\C{{\mathbf{C}}}
\def\V{{\mathbf{V}}}

\def\W{{\mathbf{W}}}
\def\U{{\mathbf{U}}}

\def\P{{\mathbf{P}}}

\def\bE{{\mathbf{E}}}

\newcommand{\gap}{\delta}
\newcommand{\fprob}{p}

\newcommand{\mD}{\mathcal{D}}
\newcommand{\E}{\mathbb{E}}

\newcommand{\trace}{\textrm{Tr}}
\newcommand{\reals}{\mathbb{R}}

\newcommand{\sign}{\mathrm{sign}}

\newcommand{\be}{\mathbf{e}}
\newcommand{\bx}{\mathbf{x}}
\newcommand{\bw}{\mathbf{w}}

\newcommand{\bu}{\mathbf{u}}
\newcommand{\bv}{\mathbf{v}}

\newcommand{\Ocal}{\mathcal{O}}

\newcommand{\norm}[1]{\|#1\|}
\newcommand{\inner}[1]{\langle#1\rangle}
\newcommand{\lemref}[1]{Lemma~\ref{#1}}

\newcommand{\thmref}[1]{Thm.~\ref{#1}}

\title{Communication-efficient Algorithms for \\Distributed Stochastic Principal Component Analysis}

\date{}
\author{Dan Garber \\Technion - Israel Institute of Technology \\ \texttt{dangar@technion.ac.il}
\and
Ohad Shamir \\ Weizmann Institute of Science \\ \texttt{ohad.shamir@weizmann.ac.il}
\and
Nathan Srebro \\ Toyotoa Technological Institute at Chicago \\ \texttt{nati@ttic.edu}
}

\begin{document}

 \maketitle
 
\begin{abstract}
We study the fundamental problem of \textit{Principal Component Analysis} in a 
statistical distributed setting in which each machine out of $m$ stores a 
sample of $n$ points sampled i.i.d. from a single unknown distribution. We 
study algorithms for estimating the leading principal component of the 
population covariance matrix that are both communication-efficient and achieve 
estimation error of the order of the centralized ERM solution that uses all 
$mn$ samples. On the negative side, we show that in contrast to results 
obtained for distributed estimation under convexity assumptions, for the PCA 
objective, simply averaging the local ERM solutions cannot guarantee error that 
is consistent with the centralized ERM. We show that this unfortunate phenomena 
can be remedied by performing a simple correction step which correlates between 
the individual solutions, and provides an estimator that is consistent with the 
centralized ERM for sufficiently-large $n$. 
We also introduce an iterative distributed algorithm that is applicable in any regime of $n$, which is based on distributed matrix-vector products. The algorithm gives significant acceleration in terms of communication rounds over previous distributed algorithms, in a wide regime of parameters.
 \end{abstract} 

\section{Introduction}
Principal Component Analysis (PCA) \cite{Pearson, Hotelling, Jolliffe2002} is one of the most celebrated and popular techniques in data analysis and machine learning.
For data that consists of $N$ vectors in $\mathbb{R}^d$, $\x_1,...,\x_N$, with normalized covariance matrix $\hat{\X} = \frac{1}{N}\sum_{i=1}^N\x_i\x_i^{\top}$, The PCA method finds the $k$-dimensional subspace (which corresponds to the span of the top $k$ principal components) such that the projection of the data onto the subspace has largest variance, i.e., it is the solution to the optimization problem:
\begin{equation}\label{eq:intro:PCA}
\max_{\W \in \reals^{d \times k} \ , \W^T \W = \I} \| \hat{\X}\W \|_F^2 .
\end{equation}
PCA is often considered in a statistical setting in which the assumption is that the input vectors are not arbitrary but sampled i.i.d. from some fixed but unknown distribution with certain general characteristics $\mD$. Then, it is often of interest to use the observed sample to estimate the top $k$ principal components of the population covariance matrix, rather then that of the sample, which leads to the modified optimization problem:
\begin{equation}\label{eq:intro:statPCA}
\max_{\W \in \reals^{d \times k} \ , \W^T \W = \I}\|  \E_{\x\sim\mD}\left[{\x\x^{\top}}\right]\W \|_F^2.
\end{equation}
Of course the empirical estimation problem \eqref{eq:intro:PCA} and the 
population estimation problem \eqref{eq:intro:statPCA} are well connected, and 
it is well-known that under mild assumptions on the distribution $\mD$ and 
given a sufficiently large sample, we can guarantee small estimation error in 
\eqref{eq:intro:statPCA} by solving optimization problem \eqref{eq:intro:PCA}.

In this work we consider the problem of estimating the first principal component (i.e., $k=1$) in a statistical and distributed setting. We assume the availability of $m$ machines, each of which stores a sample of $n$ vectors sampled i.i.d from a fixed distribution $\mD$ over $\reals^d$, and we are interested in algorithms that can be applied efficiently to solve Problem \eqref{eq:intro:statPCA} for $k=1$, with estimation error that approaches that of a \textit{centralized} algorithm, which has access to all $mn$ samples and does not pay for communication between machines. Indeed, when considering the efficiency of algorithms, we will mainly focus on the amount of communication between machines they require, since this is often the most expensive resource in distributed computing. We note that the i.i.d. assumption is standard in many applications of PCA, and can be leveraged to get more efficient algorithms than when the data partition is arbitrary. Also, we will make a standard assumption that the population covariance matrix has a non-zero additive gap between the first and second eigenvalues, which makes the problem of estimating the leading principal component meaningful.

A main challenge that often arises in many computational settings of principal 
components is that it leads to inherently non-convex optimization problems. 
While many times these problems turn out to admit efficient algorithms, 
the rich toolbox of optimization and statistical estimation procedures 
developed for \textit{convex} problems often cannot be directly 
applied to problems such as \eqref{eq:intro:PCA} and \eqref{eq:intro:statPCA}. 
Instead, 
one often needs to consider a specialized and more involved analysis, to get 
analogous convergence results for the PCA problem. This for instance was the 
case in a recent wave of results that applied concepts such as stochastic 
gradient updates \cite{Balsubramani13, Shamir16a, jain2016matching, ZhuL16c} 
and variance reduction \cite{Shamir15, shamir2016fast, Garber15, Garber16, 
ZhuL16} to the PCA problem. This is also the case in our distributed setting. 
For instance, \cite{Zhang13} proposed communication-efficient algorithms for a 
distributed statistical estimation settings, similar to ours, but under 
convexity assumptions. The authors show that under their assumptions, in a wide 
regime of parameters (namely when the per-machine sample size $n $ is large 
enough), then a simple averaging of the \textit{empirical risk minimizers} 
(ERM), computed locally on each machine, leads to estimation error of the 
population parameters of the order the centralized  ERM solution. While 
averaging makes perfect sense in a convex setting, it is clear that it can 
completely fail in a non-convex setting. Indeed, we show that already for the 
PCA problem with $k=1$, simply averaging the local ERM solutions (and 
normalizing to obtain a unit vector as required), cannot improve significantly 
over the estimation error of any single machine. We then show that a simple fix 
to the above scheme, namely correlating the directions of individual ERM 
solutions, remedies this phenomena and results in estimation error 
similar to that of the centralized ERM solution. Much like the results of 
\cite{Zhang13}, this result only holds in the regime when the per-machine 
sample size $n$ is sufficiently large. As discussed, due to the inherent 
non-convexity of the PCA objective, this approach requires a novel analysis 
tailored to the PCA problem. In this context, we view this work as an 
initiation of a research effort to understand how to efficiently aggregate 
statistical estimators in a distributed non-convex setting. 

A second line of results for distributed estimation under convexity assumptions 
consider iterative algorithms that perform multiple communication rounds and 
are based on distributed gradient computations (some examples include 
\cite{Shamir14, Disco, LeeML15, 
Shamir16b,jaggi2014communication,reddi2016aide}). The benefit of these methods 
is that (a) they provide meaningful estimation error guarantees in a much wider 
regime of parameters than the ``one-shot" aggregation methods (namely in terms 
of the number of samples per machine), and (b), due to their iterative nature, 
they allow to approximate the centralized ERM solution arbitrary well. 
Unfortunately, these methods, all of which rely heavily on convexity 
assumption, cannot be directly applied to the PCA problem. Towards designing 
efficient distributed iterative methods for our PCA setting, we consider the 
application of the recently proposed method of \textit{Shift-and-Invert power 
iterations} (S\&I) for PCA \cite{Garber15, Garber16}. The S\&I method reduces 
the problem of computing the leading eigenvector of a real positive 
semidefinite matrix to that of approximately solving a small number (i.e. 
poly-logarithmic in the problem parameters) of systems of linear equations. 
These in turn, could be efficiently solved by arbitrary distributed convex 
solvers. We show that coupling the S\&I method with the stochastic 
pre-conditioning technique for linear systems proposed in \cite{Disco} and well 
known fast gradient methods such as the conjugate gradient method, 
gives state-of-the-art guarantees in terms of communication costs, and provides 
a significant improvement over distributed variants of classical fast 
eigenvector algorithms such as power iterations and the faster Lanczos 
algorithm. Much like its convex counterparts, which only rely on distributed 
gradient computations and simple vector aggregations, our iterative method only 
relies on distributed matrix-vector products, i.e., it requires each machine to 
only send products of its local empirical covariance matrix with some input 
vector.

Beyond the results described so far,  
\cite{liang2014improved,boutsidis2016optimal} studied distributed algorithms 
for PCA in a \textit{deterministic} setting in which the partition of the data 
across machines is arbitrary and communication is measured in terms of number 
of transmitted bits. The approximation guarantees provided in these works are 
in terms of the projection of the data onto the leading principal components 
(instead of alignment between the estimate and the optimal solution, studied in 
this paper). Applying these results to our setting will give a number of 
communication rounds that scales like 
$\textrm{poly}(\epsilon^{-1}\delta^{-1})$, where $\epsilon$ is the desired 
error and $\delta$ is the population eigengap. In our setting, $\epsilon$ will 
scale with the inverse of the size of the sample, i.e., $\epsilon \approx 
(mn)^{-1}$, which for these algorithms will result in amount of communication 
that is polynomial in the size of the data. In contrast, we will be interested 
in algorithms whose communication costs does not scale with $n$ at all.
In this context we note that, by focusing on algorithms that either perform simple aggregation of local ERM solutions, or perform only distributed matrix-vector products with the empirical covariance matrix, we can circumvent the need to measure communication explicitly in terms of the number of bits transmitted, which often burdens the analysis of natural algorithms, such as those proposed here.

\section{Preliminaries}

\subsection{Notation and problem setting}
We write vectors in $\reals^d$ in boldface lower-case letters (e.g., $\v$), matrices in boldface upper-case letters (e.g., $\X$), and scalars are written as lightface letters (e.g., $c$). We let $\Vert\cdot\Vert$ denote the standard Euclidean norm for vectors and the spectral norm for matrices.

We consider the following statistical distributed setting.
Let $\mD$ be a distribution over vectors in $\reals^d$ with squared $\ell_2$ norm at most $b$, for some $b>0$. We consider a setting in which $m$ machines, numbered $1...m$, are each given a dataset of $n$ samples drawn i.i.d. from $\mD$.
We let $\v_1$ denote a leading eigenvector of the population covariance matrix $\X = \E_{\x\sim\mD}[\x\x^{\top}]$. Our goal is to efficiently (mainly in terms of communication) find an estimate $\w$ for $\v_1$, i.e., a unit vector that maximizes the product $(\v_1^{\top}\w)^2$ with high probability. Towards this end, we assume that the population covariance matrix $\X$ has a non-zero eigengap $\gap$, i.e., $\gap := \lambda_1(\X) - \lambda_2(\X) > 0$, where $\lambda_i(\cdot)$ denotes the $i$th largest eigenvalue of a symmetric real matrix. Note that $\gap>0$ is necessary for $\v_1$ to be uniquely defined (up to sign). 

In addition, we let $\hat{\X}_i$ denote the empirical covariance matrix of the 
sample stored on machine $i$ for every $i\in[m]$, i.e., $\hat{\X}_i = 
\frac{1}{n}\sum_{j=1}^n\x_j^{(i)}\x_j^{(i)\top}$, where 
$\x_1^{(i)}...\x_n^{(i)}$ are the samples stored on machine $i$. We let 
$\hat{\X}$ denote the empirical covariance matrix of the union of points across 
all machines i.e., $\hat{\X} = \frac{1}{m}\sum_{i=1}^m\hat{\X}_i$.

Our model of communication assumes that the $m$ machines work in rounds during 
which a central machine (w.l.o.g. machine 1) can send a single vector in 
$\reals^d$ to all other machines, or every machine can send either the leading 
eigenvector of its local empirical covariance matrix, or the product of a 
single input vector with its local covariance, to machine 1. We will measure 
communication complexity in terms of number of such rounds required to achieve 
a certain estimation error.

\subsubsection{The centralized solution}

Our primary benchmark for measuring performance will be the \textit{centralized empirical risk minimizer} which is the leading eigenvector of the aggregated empirical covariance matrix $\hat{\X}$.

The following standard result bounds the error of the centralized ERM.

\begin{lemma}[Risk of centralized ERM]\label{lem:ERM4PCA}
Fix $p\in(0,1)$. Suppose that $\delta > 0$ and let $\hat{\v}_1$ denote the leading eigenvector of $\hat{\X}$, i.e., $\hat{\v}_1\in\arg\max_{\v:\Vert{\v}\Vert=1}\v^{\top}\hat{\X}\v$. Then it holds w.p. at least $1-p$ that
\begin{equation}\label{eq:infromalResultERM}
1 - (\v_1^{\top}\hat{\v}_1)^2 ~\leq~ \epsilon_{\textrm{ERM}}(p) := \frac{32b^2\ln(d/p)}{mn\delta^2} .
\end{equation}
\end{lemma}
Lemma \ref{lem:ERM4PCA} is a direct consequence of the following standard 
concentration argument for random matrices, and the Davis-Kahan sin($\theta$) 
theorem (whose proof is given in the appendix for completeness):

\begin{theorem}[Matrix Hoeffding, see \cite{Tropp12}]\label{thm:matHoff}
Let $\mD$ be a distribution over vectors with squared $\ell_2$ norm at most $b$, and let $\X = \E_{\x\sim\mD}[\x\x^{\top}]$. Let $\hat{\X} = \frac{1}{n}\sum_{i=1}^n\x_i\x_i^{\top}$, where $\x_1,...,\x_n$ are sampled i.i.d. from $\mD$. Then, it holds that $$\forall \epsilon > 0: \quad \Pr\left({\Vert{\hat{\X}-\X}\Vert\geq \epsilon}\right) ~\leq~ d\cdot\exp\left({-\frac{\epsilon^2n}{16b^2}}\right).$$ 
\end{theorem}

\begin{theorem}[Davis-Kahan sin($\theta$) theorem]\label{thm:DK}
Let $\X,\Y$ be symmetric real $d\times d$ matrices with leading eigenvectors $\v_{\X}$ and $\v_{\Y}$ respetively. Also, suppose that $\delta(\X) := \lambda_1(\X) - \lambda_{2}(\X) > 0$. Then it holds that $$1 - \left({\v_{\X}^{\top}\v_{\Y}}\right)^2  ~\leq~ 2\frac{\Vert{\X-\Y}\Vert^2}{\delta(\X)^2}.$$

\end{theorem}

\subsection{Informal statement of main results and previous algorithms}

We now informally describe our main results, followed by a detailed description of previous approaches that are directly applicable to our setting. The algorithmic results (both new and old) are summarized in Table \ref{table:results}.

\begin{table*}\renewcommand{\arraystretch}{1.3}
{\footnotesize
\begin{center}
  \begin{tabular}{| l | c | c |}    \hline
    Method &  $1-(\w^{\top}\v_1)^2$ w.p. $3/4$ & \# communcation rounds \\ \hline 
    Centralized ERM & $ \epsilon_{\textrm{ERM}} = \Theta(\frac{b^2\ln{d}}{\delta^2mn})$ & - \\ \hline    
    Distributed Power Method & $\epsilon_{\textrm{ERM}}\cdot(1+o(1))$ & $\tilde{O}(\lambda_1/\delta)$ \\ \hline
    Distributed Lanczos & $\epsilon_{\textrm{ERM}}\cdot(1 + o(1))$ & $\tilde{O}(\sqrt{\lambda_1/\delta})$    \\ \hline
    ``Hot-potato" SGD  & $O(\epsilon_{\textrm{ERM}})$ & $m$  \\ \hline
    Average of ERMs with sign-fixing (Thm. \ref{thm:ermavg})& $ O(\epsilon_{\textrm{ERM}}) + O\left({\frac{b^4\ln^2{d}}{\delta^4n^2}}\right)$ & $1$ \\ \hline
    Dist. Shift\&Invert + precond. linear sys. (Thm. \ref{thm:shiftNinvert:main}) & $\epsilon_{\textrm{ERM}}\cdot(1 + o(1))$ & $\tilde{O}(\min\{(b/\delta)^{1/2}n^{-1/4},~ m^{1/4}\})$ \\ \hline
  \end{tabular}
  \caption{Comparison of estimation error and number of communication rounds. For simplicity we fix the failure probability to $p = 1/4$ and assume $mn$ is in the regime in which Lemma \ref{lem:ERM4PCA} is meaningful, i.e, $mn = \Omega(b^2\delta^{-2}\ln{d})$. The $\tilde{O}(\cdot)$ suppresses logarithmic factors in $b,d,1/p, 1/\epsilon_{\textrm{ERM}}$. For the result of Theorem  \ref{thm:ermavg} we assume the regime $m=O(d)$. The sub-constant $o(1)$ factors could be made, in principle, arbitrary small in all relevant results by trading approximation with communication.
  }\label{table:results}
\end{center}}
\end{table*}\renewcommand{\arraystretch}{1} 

\subsubsection{Main results}

\paragraph{Failure of simple averaging of local ERM solutions} We show that a natural approach of simply averaging the individual leading eigenvectors of the empirical covariance matrices $\hat{\X}_i$ (and normalizing the obtain a unit vector) cannot significantly improve (beyond logarithmic factors) over the performance of any of the individual eigenvectors. More concretely, if we let $\hat{\v}_1^{(i)}$ denote the leading eigenvector of $\hat{\X}_i$ for any $i\in[m]$, and we denote their average by $\bar{\v}_1 = \frac{1}{m}\sum_{i=1}^m\hat{\v}_1^{(i)}$, then there exists a distribution $\mD$ over vectors with magnitude $O(1)$ and covariance eigengap $\delta =1$, such that
\begin{eqnarray*}
\forall m,n:\quad \E_{\mD}\left[{1 - 
\left({\frac{\bar{\v}_1^{\top}\v_1}{\Vert{\bar{\v}_1}\Vert}}\right)^2}\right] = 
\Omega\left({\frac{1}{n}}\right),
\end{eqnarray*}
See Theorem \ref{thm:simpleAvgFail} in Section \ref{sec:singleRoundAlgs} for the complete and formal argument.

\paragraph{A successful single communication round algorithm via correlation of individual ERM solutions}
We show that if prior to averaging the local ERM solutions, as suggested above, we correlate their directions by aligning them according to any single machine (say machine number 1), i.e., we let $\bar{\v}_1 = \frac{1}{m}\sum_{i=1}^m\sign(\hat{\v}_1^{(i)\top}\hat{\v}_1^{(1)})\hat{\v}_1^{(i)}$, then this guarantees that for any $p\in(0,1)$, w.p. at least $1-p$,
\begin{eqnarray}\label{eq:infromalResultSingle}
1 - \left({\frac{\bar{\v}_1^{\top}\v_1}{\Vert{\bar{\v}_1}\Vert}}\right)^2= O\left(\frac{b^2\ln\left(\frac{dm}{\fprob}\right)}{\gap^2 mn}+\frac{ b^4 \ln^2\left(\frac{dm}{\fprob}\right)}{\gap^4 n^2}\right) .
\end{eqnarray}

See Theorem \ref{thm:ermavg} in Section \ref{sec:singleRoundAlgs} for the complete and formal result.

In particular, in the likely scenario when $m = O(d/p)$ we have that w.p. at least $1-p$, $1 - \left({\bar{\v}_1^{\top}\v_1/\Vert{\bar{\v}_1}\Vert}\right)^2 = \epsilon_{\textrm{ERM}}(p))\cdot O\left({1 + m^2\cdot\epsilon_{\textrm{ERM}}(p)}\right)$ ,
where $ \epsilon_{\textrm{ERM}}(p))$ is defined in Eq. 
\eqref{eq:infromalResultERM}. Another related interpretation of the results
is that the bound in Eq. \eqref{eq:infromalResultSingle} is comparable with 
$\epsilon_{\textrm{ERM}}$ (up to poly-log factors) when $n = 
\Omega\left({\delta^{-2}b^2m\ln(dm/p)}\right)$.

We also show a matching lower bound that the bound in Eq. \eqref{eq:infromalResultSingle} is tight (up to poly-log factors) for this aggregation method, see Theorem \ref{thm:signFixLB}.

\paragraph{A multi communication round algorithm}
We present a distributed algorithm based on the \textit{Shift-and-Invert} 
framework for leading eigenvector computation \cite{Garber15, Garber16} which 
is applied to explicitly solving the centralized ERM problem. We show that for 
any $p\in(0,1)$, when $mn = \Omega(b^2\ln(d/p)/\delta^2)$ (i.e., when Lemma 
\ref{eq:infromalResultERM} is meaningful), the algorithm produces a solution 
$\w$ such that w.p. at least $1-p$,
\begin{eqnarray}\label{eq:infromalResultMulti}
1 - (\v_1^{\top}\w)^2 ~\leq~\epsilon_{\textrm{ERM}}(p))\cdot \left({1 + o(1)}\right),
\end{eqnarray}
where $ \epsilon_{\textrm{ERM}}(p))$ is defined in Eq. \eqref{eq:infromalResultERM}.
The algorithm performs overall $\tilde{O}(\sqrt{b}\delta^{-1/2}n^{-1/4})$ distributed matrix-vector products with the centralized empirical covariance matrix $\hat{\X}$ \footnote{i.e., on each round, each machine $i$ sends the product of an input vector in $\reals^d$ with its local covariance matrix $\hat{\X}_i$.}.
The $\tilde{O}(\cdot)$ notation hides poly-logarithmic factors in $1/p, 1/\gap, d, 1/\epsilon_{\textrm{ERM}}(p)$. 
See Theorem \ref{thm:shiftNinvert:main} in Section \ref{sec:multiRoundAlgs} for the complete and formal result.

We note that in particular, under our assumption that $mn = \tilde{\Omega}(b^2/\gap^2)$, it holds that the number of distributed matrix-vector products is upper bounded by $\tilde{O}(m^{1/4})$.
Moreover, in the regime $n = \Omega(b^2\delta^{-2})$, we can see that the number of distributed matrix-vector products depends only poly-logarithmically on the problem parameters.

In general, the sub-constant $o(1)$ factor in \eqref{eq:infromalResultMulti} could be made arbitrarily small by trading the approximation error with the number of distributed matrix-vector products.

\subsubsection{Previous algorithms}\label{sec:relatedWork}

\paragraph{Distributed versions of classical iterative algorithms:}
Classical fast iterative algorithms for computing the leading eigenvector of a 
positive semidefinite matrix, such as the well-known Power Method and the 
Lanczos Algorithm, require iterative multiplications of the input matrix 
($\hat{\X}$ 
in our case) with the current estimate. It is thus straightforward to 
implement these algorithms in our distributed setting, by multiplying the same 
vector with the covariance matrices at each machine, and averaging the result. 
Thus, by well-known 
convergence guarantees of these two methods, we will have that for a fixed 
$\epsilon > 0$, these methods produce a  unit vector $\w$ such that, for any 
$p\in(0,1)$, $1 - (\w^{\top}\hat{\v}_1)^2 ~\leq~ \epsilon $ w.p. at least $1-p$,
after $O(\hat{\lambda}_1\hat{\delta}^{-1}\ln(d/p\epsilon))$ rounds for the 
Power Method and $O(\sqrt{\hat{\lambda}_1\hat{\delta}^{-1}}\ln(d/p\epsilon))$ 
for the Lanczos Algorithm, where $\hat{\lambda}_1,\hat{\delta}$ denote the 
leading eigenvalue and eigengap of $\hat{\X}$, respectively. Moreover, in 
the regime of $mn$ in which Lemma \ref{lem:ERM4PCA} is meaningful, we can 
replace $\hat{\lambda}_1,\hat{\delta}$ with $\lambda_1,\delta$ in the above 
bounds, and the result will still hold with high probability. 

Simple calculations show that in the regime of $mn$ in which Lemma 
\ref{lem:ERM4PCA} is meaningful, it holds that our Shift-and-Invert-based 
algorithm outperforms distributed Lanczos (in terms of worst-case guarantees) 
whenever $n= 
\tilde{\Omega}(b^2/\lambda_1^2)$.

\paragraph{``Hot potato" SGD:}
Another straightforward approach is to apply a sequential algorithm for direct risk minimization that can process the data-points one by one, such as stochastic gradient descent (SGD), by passing its state from one machine to the next, after completing a full pass over the machine's data. Clearly, this process of making a full pass over the data of a certain machine before sending the final estimate to the next one, requires overall $m$ communication rounds in order to make a full pass over all $mn$ points.
SGD for PCA was studied in several results in recent years \cite{Balsubramani13, Shamir16a, shamir2016fast, Jain16, ZhuL16c}. For instance applying the result of \cite{Jain16} in this way will result in a final estimate $\w$ satisfying
\begin{eqnarray}\label{eq:hotpotSGD}
1 - (\w^{\top}\v_1)^2 ~=~ O\left({\frac{b^2\ln{d}}{\delta^2mn}}\right) \quad \textrm{w.p. at least } ~ 3/4.
\end{eqnarray}
We note that in the regime in which the bound in \eqref{eq:hotpotSGD} is meaningful it holds that the number of communication rounds of our Shift-and-Invert-based algorithm is upper-bounded by $\tilde{O}(m^{1/4})$ which for sufficiently large $m$ dominates the communication complexity of SGD.

\section{Single Communication Round Algorithms via ERM on Each Machine}\label{sec:singleRoundAlgs}

In this section we consider distributed algorithms that require only a single 
round of communication. Naturally for this regime, all algorithms will be based 
on aggregating the ERM solutions of the individual machines, i.e., each machine 
$i$ only sends the leading eigenvector of its empirical covariance matrix 
$\hat{\X}_i$ to a centralized machine (without loss of generality, machine 1) 
which it turn combines them to a single unit vector in some manner.

\subsection{Simple averaging of eigenvectors fail}

Perhaps the simplest method to aggregate the individual eigenvectors of each  
machine is to average them, and then normalize to obtain a unit vector.
For instance, in the distributed statistical setting considered in 
\cite{Zhang13}, in which the objective is \textit{strongly convex}, it was 
shown that simply averaging the individual ERM solutions leads, in a meaningful 
regime of parameters, to estimation error of the order of the centralized ERM 
solution. However, here we show that for PCA, in which the objective is 
certainly not convex, this approach fails practically in any regime, in the 
sense that the error of the returned aggregated solution can be no better than 
that returned by any single machine. 

\begin{theorem}\label{thm:simpleAvgFail}
There exists a distribution over vectors in $\reals^2$ with $\ell_2$ norm bounded by a universal constant for which the eigengap in the covariance matrix is 1 (i.e., $\delta=1$), such that if each machine $i$ returns an estimate $\hat{\v}_1^{(i)}$ which is an unbiased leading eigenvector of $\hat{\X}_i$ (i.e., both outcomes $-\hat{\v}_1^{(i)}, + \hat{\v}_1^{(i)}$ are equally likely), then the aggregated vector $\bar{\v}_1 = \frac{1}{m}\sum_{i=1}^m\hat{\v}_1^{(i)}$ satisfies
\begin{eqnarray*}
\forall m,n: \quad \E\left[{1 - 
\left\langle{\frac{\bar{\v}_1}{\Vert{\bar{\v}_1}\Vert}, 
\v_1}\right\rangle^2}\right] ~=~ \Omega(1/n) .
\end{eqnarray*}
\end{theorem}
The proof is given in the appendix.
 
\subsection{Averaging with Sign Fixing}

As evident from the statement of Theorem \ref{thm:simpleAvgFail}, 
an important assumption is that each machine produces an unbiased estimate, in 
the sense that the sign of the outcome is uniform and independent of the other 
machines. This hints that correlating the signs of the 
different estimates can circumvent the lower bound result in Theorem 
\ref{thm:simpleAvgFail}. It turns out that this is indeed the case, as captured 
by the following theorem:




\begin{theorem}\label{thm:ermavg}
Let $\tilde{\w}_i$ be the leading eigenvector of $\hat{\X}_i$ for any $i\in[m]$, and consider the unit vector
\begin{eqnarray}\label{eq:signfixSol}
\w = \frac{\sum_{i=1}^m\sign(\tilde{\w}_i^{\top}\tilde{\w}_1)\tilde{\w}_i}{\Vert{\sum_{i=1}^m\sign(\tilde{\w}_i^{\top}\tilde{\w}_1)\tilde{\w}_i}\Vert}.
\end{eqnarray}
Then, for any $p\in(0,1)$,  it holds w.p. at least $1-p$ that
	\[
	1-(\v_1^{\top}\w)^2 ~=~ 
	O\left(\frac{b^2\log\left(\frac{dm}{\fprob}\right)}{\gap^2 mn}+\frac{ b^4 \log^2\left(\frac{dm}{\fprob}\right)}{\gap^4 n^2}\right) .
	\]
\end{theorem}

For ease of presentation, throughout the rest of this section we denote the correlated vector $\hat{\w}_i = \sign(\tilde{\w}_i^{\top}\tilde{\w}_1)\tilde{\w}_i$ for any $i\in[m]$.

The main step towards proving Theorem \ref{thm:ermavg} is to consider each 
$\hat{\w}_i$ as an approximately unbiased perturbation of the true leading 
eigenvector $\v_1$ and to upper bound the magnitude of this perturbation. This 
is carried out in the following much more general and self-contained lemma, 
which might be of independent interest. 

\begin{lemma}\label{lem:taylor}
Let $\A$ be a positive semidefinite matrix with some fixed leading eigenvector $\bv_1$, a leading eigenvalue $\lambda_1$ and an eigengap $\gap :=\lambda_1(\A) - \lambda_2(\A) > 0$. Let $\hat{\A}$ be some positive semidefinite matrix such that $\norm{\hat{\A}-\A}\leq \gap/4$. Then there is a unique leading eigenvector $\hat{\bv}_1$ of $\hat{\A}$ such that $\inner{\hat{\bv}_1,\bv}\geq 0$, and
	\[
	\left\|\hat{\bv}_1-\bv_1-(\lambda_1\I -\A)^\dagger(\hat{\A}-\A) \bv_1\right\|~\leq~ \frac{c\norm{\hat{\A}-\A}^2}{\gap^2},
	\]
	where $\dagger$ denotes the pseudo-inverse, and $c$ is a positive numerical constant.
\end{lemma}

\begin{proof}
	The proof is based on viewing $\hat{\A}$ as an unbiased 
	perturbation of the matrix $\A$, and computing a Taylor expansion of 
	$\hat{\bv}_1$ around $\bv_1$. 
	For notational convenience, let $\bE=\hat{\A}-\A$, and define $\A(t)=\A+t\bE$ for $t\in [0,1]$. Also, define $\lambda(t)$ to be the leading eigenvalue of $\A(t)$. 
	
	First, we note that for any $t\in [0,1]$, $\A(t)$ has an eigengap of at least $\gap/2$ between its first two eigenvalues (since by Weyl's inequality, its eigenvalues are at most $\norm{t\bE}\leq \norm{\bE}\leq \gap/4$ different than $\A$, and we know that $\A$ has an eigengap of $\gap$). Therefore, the leading eigenvalue of $\A(t)$ is simple. This means that the function $\bv(t)$, which equals the leading eigenvector of $\A(t)$, is uniquely defined up to a sign. This sign will be chosen so that $\inner{\bv(t),\bv_1}\geq 0$, which makes $\bv(t)$ unique and well-defined\footnote{Note that ties are impossible, since that can only happen if $\inner{\bv(t),\bv_1}=0$, yet by applying the Davis-Kahan sin($\theta$) theorem (Theorem \ref{thm:DK}), $\sqrt{1-\inner{\bv(t),\bv_1}^2} \leq \frac{2\norm{\A(t)-\A}}{\gap}\leq \frac{2\norm{\bE}}{\gap}\leq \frac{1}{2}$.}.
	By Theorem 1 in \cite{magnus1985differentiating}, we have that both $\lambda(t)$ and $\bv(t)$ are infinitely differentiable at any $t\in [0,1]$, and satisfy\footnote{Formally speaking, the theorem only ensures $\bv(t),\lambda(t)$ exist and are infinitely differentiable in some open neighborhood of $t$. However, since the result holds for any $t\in [0,1]$, and the proof implies that these functions are unique in each such neighborhood (where the uniqueness of $\bv(t)$ holds once we fixed the sign as above), it follows that the same holds in all of $t\in [0,1]$.}
	\[
		\lambda'(t) = \bv(t)^\top \bE \bv(t)~~,~~ \bv'(t) = (\lambda(t)\I-\A(t))^{\dagger}\bE\bv(t)~.
	\]
	We will also need to bound the second derivative of $\bv(t)$. By the product rule and the equations above, this derivative equals
	\begin{align}
	\bv''(t)&= \frac{\partial}{\partial t}\left((\lambda(t)\I-\A(t))^{\dagger}\right)\bE\bv(t)+
	(\lambda(t)\I-\A(t))^{\dagger}\bE\frac{\partial}{\partial t}\bv(t)\notag\\
	&= \frac{\partial}{\partial t}\left((\lambda(t)\I-\A(t))^{\dagger}\right)\bE\bv(t)+
	(\lambda(t)\I-\A(t))^{\dagger}\bE(\lambda(t)\I-\A(t))^{\dagger}\bE\bv(t).\label{eq:v2}
	\end{align}
	To compute the derivative above, we apply the chain rule. The derivative of a pseudo-inverse $\B^{\dagger}$ of a matrix-valued function $\B=\B(t)$ with respect to $t$ (assuming $\B$ and hence its pseudo-inverse is symmetric for all $t$) is given by (see Theorem 4.3 in \cite{golub1973differentiation})
	\[
	-\B^\dagger\left( \frac{\partial}{\partial t} \B\right) \B^\dagger+\left(\B^\dagger\right)^2\left( \frac{\partial}{\partial t} \B\right) (I-\B\B^\dagger)+(\I-\B^\dagger \B)\left( \frac{\partial}{\partial t} \B\right)\left(\B^\dagger\right)^2~.
	\]
	This formula is true assuming the rank of $\B$ is constant in some open neighborhood of $t$. Applying this for $\B = \lambda(t)\I-\A(t)$ (which indeed has a fixed rank of $d-1$ by the eigengap assumption), noting that $\left\|\frac{\partial}{\partial t}(\lambda(t)\I-\A(t))\right\|=\left\|\bv(t)^\top \bE \bv(t)\I-\bE\right\|\leq 2\norm{\bE}$, and using the facts that $\norm{\bv(t)}=1$, $\norm{\I-\B^\dagger \B}\leq 1$,$\norm{\I-\B\B^\dagger}\leq 1$ and $\norm{(\lambda(t)\I-\A(t))^\dagger}\leq 2/\gap$ (since the smallest non-zero eigenvalue of $\lambda(t)\I-\A(t)$ is at least $\gap/2$), we have that
	\[
	\left\|\frac{\partial}{\partial t}\left((\lambda(t)\I-\A(t))^{\dagger}\right)\right\|~\leq~ \frac{24\cdot\norm{\bE}}{\gap^2}.
	\]
	Plugging this into \eqref{eq:v2}, and again using the fact that $\norm{(\lambda(t)\I-\A(t))^\dagger}\leq 2/\gap$, we get that
	\[
	\left\|\bv''(t)\right\| \leq \frac{c\norm{\bE}^2}{\gap^2}
	\]
	for some numerical constant $c$. 
	
	By a first-order Taylor expansion of $\bv(t)$ with an explicit remainder term\footnote{Since $\bv(t),\bv'(t),\bv''(t)$ are all vectors, this is a direct consequence of the standard Taylor expansion of the scalar function $t\mapsto v(t)_j$, mapping $t$ to the $j$-th coordinate of $\bv(t)$, using the fact that this mapping is differentiable to any order (see Theorem 1 in \cite{magnus1985differentiating}, and in particular twice continuously differentiable.},
	\[
	\bv(1) = \bv(0)+\bv'(0)+\frac{1}{2}\int_{t=0}^{1}(1-t)^2\bv''(t)dt~,
	\]
	which by the equations above and the definition of $\bv(t)$ implies that
	\[
	\hat{\bv}_1 = \bv_1+(\lambda_1 \I-\A)^\dagger \bE \bv_1 + \frac{1}{2}\int_{t=0}^{1}(1-t)^2\bv''(t)dt~.
	\]
	This implies
	\[
	\left\|\hat{\bv}_1-\bv_1-(\lambda_1\I-\A)^\dagger \bE \bv_1\right\|~\leq~ \frac{1}{2}\int_{t=0}^{1}(1-t)^2\norm{\bv''(t)}dt
	~\leq~ \frac{c\norm{\bE}^2}{2\lambda^2}\int_{t=0}^{1}(1-t)^2 dt,
	\]
	which is at most $c'\norm{\bE}^2/\lambda^2$ for some appropriate numerical constant $c'$. Plugging back $\bE=\hat{\A}-\A$, the result follows.
\end{proof}

Lemma \ref{lem:taylor} is central to the proof of the following Lemma, of which 
the proof of Theorem \ref{thm:ermavg} is an easy consequence. 

\begin{lemma}\label{lem:eigboundtaylor}
	The following two conditions hold with probability at least $1-\fprob-d\exp(-\gap^2 n/cb^2)$, for some numerical constants $c,c'>0$:
	\begin{itemize}
	\item The leading eigenvalue of every $\hat{\X}_i$ is simple, i.e., $\lambda_1(\hat{\X}_i) - \lambda_2(\hat{\X}_i) > 0$.
		\item Fixing $\bv_1$, there exist unique leading eigenvectors $\hat{\bv}^i_1,\ldots,\hat{\bv}^i_m$  of $\hat{\X}_1,\ldots,\hat{\X}_m$, such that $\max_i \norm{\hat{\bv}^i_1-\bv_1}\leq \frac{1}{4}$, and $$\Big\|\frac{1}{m}\sum_{i=1}^{m}\hat{\bv}^i_1-\bv_1\Big\| ~\leq~
	c'\Big(\frac{b^2\log(2dm/\fprob)}{\gap^2 n}+
	\sqrt{\frac{b^2\log(2dm/\fprob)}{\gap^2 mn}}\Big).$$
	\end{itemize}
\end{lemma}

\begin{proof}
	Using the matrix Hoeffding inequality (Theorem \ref{thm:matHoff}) and a union bound, we that
	\begin{equation}\label{eq:badev}
	\Pr\left(\exists i,~~\norm{\hat{\X}_i-\X}> \frac{\gap}{12}\right)~\leq~
	md\exp\left(-\frac{\gap^2 n}{c' b^2}\right)
	\end{equation}
	for some constant $c'>0$. Thus, with high probability, $\max_i \norm{\hat{\X}_i-\X}\leq \gap/12$. By Weyl's inequality, it follows that the eigenvalues of $\X$ and $\hat{\X}_i$ are at most $\gap/12$ apart, and since $\X$ has an eigengap of $\gap$ between its two leading eigenvalues, it follows that $\hat{\X}_i$ has an eigengap of at least $\gap-\gap/12-\gap/12 > 0$, which proves the first part of the lemma. To handle the second part, note that by a variant of the Davis-Kahan sin$\theta$ theorem (see Corollary 1 in \cite{yu2015useful}), if $\max_i \norm{\hat{\X}_i-\X}\leq \gap/12$, then the leading eigenvectors $\hat{\bv}^i_1$ of $\hat{\X}_i$ (after choosing the sign appropriately, i.e. $\inner{\hat{\bv}^i_1,\bv_1}\geq 0$) are all at a distance of at most $1/4$ from $\bv_1$. Moreover, by \lemref{lem:taylor},
	\[
	\frac{1}{m}\sum_{i=1}^{m}\left\|\hat{\bv}^i_1-\bv_1- (\lambda_1\I-\X)^\dagger(\hat{\X}_i-\X)\bv_1\right\| ~\leq~ \frac{c}{\gap^2}\cdot\frac{1}{m}\sum_{i=1}^{m}\norm{\hat{\X}_i-\X}^2.
	\]
	By the triangle inequality, this implies
	\[
	\left\|\frac{1}{m}\sum_{i=1}^{m}\hat{\bv}^i_1-\bv_1- (\lambda_1\I-\X)^\dagger\left(\frac{1}{m}\sum_{i=1}^{m}(\hat{\X}_i-\X)\right)\bv_1\right\| ~\leq~ \frac{c}{\gap^2}\cdot\frac{1}{m}\sum_{i=1}^{m}\norm{\hat{\X}_i-\X}^2,
	\]
	and therefore (as $\norm{\bv_1}=1$),
	\begin{equation}\label{eq:vhatv}
	\left\|\frac{1}{m}\sum_{i=1}^{m}\hat{\bv}^i_1-\bv_1\right\| ~\leq~
	\frac{c}{\gap^2}\cdot\frac{1}{m}\sum_{i=1}^{m}\norm{\hat{\X}_i-\X}^2+\left\|(\lambda_1\I-\X)^\dagger\right\|\cdot\left\|\frac{1}{m}\sum_{i=1}^{m}\hat{\X}_i-\X\right\|.
	\end{equation}
	Since $\X$ has an eigengap of $\gap$, it follows that the minimal non-zero eigenvalue of $\lambda_1\I-\X$ is at least $\gap$, and therefore $\left\|(\lambda_1\I-\X)^\dagger\right\|\leq 1/\gap$. As to the other terms, recall that $\hat{\X}_i$ is the average of $n$ i.i.d. matrices with mean $\X$, and $\frac{1}{m}\sum_{i=1}^{m}\hat{\X}_i$ is the average of $mn$ such i.i.d. matrices. Thus, by a matrix Hoeffding inequality (Theorem \ref{thm:matHoff}) and a union bound, it holds with probability at least $1-\fprob$ that
	\[
	\forall i,~~\norm{\hat{\X}_i-\X} ~\leq~ c_1\sqrt{\frac{b^2\log(2dm/\fprob)}{n}}
	\]
	as well as
	\[
	\left\|\frac{1}{m}\sum_{i=1}^{m}\hat{\X}_i-\X\right\|~\leq~ c_1 \sqrt{\frac{b^2\log(2dm/\fprob)}{mn}}
	\]
	for some constant $c_1$. Combining this with \eqref{eq:badev} using a union bound, and plugging into \eqref{eq:vhatv}, it follows that with probability at least $1-\fprob-d\exp\left(-\frac{\gap^2 n}{c' b^2}\right)$, 
	\[
	\left\|\frac{1}{m}\sum_{i=1}^{m}\hat{\bv}^i_1-\bv_1\right\| ~\leq~
	\frac{cc_1^2 b^2\log(2dm/\fprob)}{\gap^2 n}+c_1\sqrt{\frac{b^2\log(2dm/\fprob)}{\gap^2 mn}}.
	\]
	Slightly simplifying, the result follows.
\end{proof}

We can now complete the proof of Theorem \ref{thm:ermavg}.

\begin{proof}[Proof of \thmref{thm:ermavg}]
	The proof is an easy consequence of \lemref{lem:eigboundtaylor}. Assuming the events in the lemma occur, we have that the leading eigenvalues of $\X$ as well as $\hat{\X}_i$ for all $i$ are simple, hence the leading eigenvectors are all unique up to a sign. In particular, let $\bv_1$ be the eigenvector closest to $\tilde{\bw}_1=\hat{\bw}_1$, with ties broken arbitrarily, so that $\norm{\hat{\bw}_1-\bv_1}\leq \norm{\hat{\bw}_1+\bv_1}$. This implies that $\hat{\bw}_1=\hat{\bv}^1_1$ (where $\hat{\bv}^1_1$ is as defined in \lemref{lem:eigboundtaylor}), since otherwise, by the inequality above, we would get $\norm{-\hat{\bv}^1_1-\bv_1}\leq \norm{-\hat{\bv}^1_1+\bv_1}$,
	which implies in turn $\inner{\hat{\bv}^1_1,\bv_1}\leq 0$, contradicting the fact that $\norm{\hat{\bv}^1_1-\bv_1}=\sqrt{2-2\inner{\hat{\bv}_1,\bv_1}}$ is at most $1/4$ by \lemref{lem:eigboundtaylor}.
	
	Having established that $\hat{\bw}_1=\hat{\bv}^1_1$, we note that by \lemref{lem:eigboundtaylor} and the triangle inequality, for any $i>1$,
	\[
	\norm{\hat{\bv}^i_1-\hat{\bv}^1_1} \leq \frac{1}{2} ~~~\text{and therefore}~~~ \norm{\hat{\bv}^i_1-\hat{\bw}_1}\leq \frac{1}{2}.
	\]
	As $\hat{\bv}^i_1,\hat{\bw}_1$ are unit vectors, this implies that $\norm{\hat{\bv}^i_1-\hat{\bw}_1}<\norm{-\hat{\bv}^i_1-\hat{\bw}_1}$. 
	Since for any $i$, we have $\hat{\bw}_i\in \{-\hat{\bv}^i_1,\hat{\bv}^i_1\}$, with the sign chosen based on which vector is closest to $\hat{\bw}_1$, it follows that $\hat{\bw}_i=\hat{\bv}^i_1$ for all $i$. Applying \lemref{lem:eigboundtaylor} with $\hat{\bw}_i=\hat{\bv}^i_1$, we get that with probability at least $1-\fprob-d\exp\left(-\gap^2 n/c b^2\right)$,
	\begin{eqnarray*}	
	\Big\|
	\frac{1}{m}\sum_{i=1}^{m}\hat{\bw}_i-\bv_1
	\Big\| ~\leq ~
c'\Big(\frac{b^2\log(2dm/\fprob)}{\gap^2 n}
+\sqrt{\frac{b^2\log(2dm/\fprob)}{\gap^2 mn}}\Big).
	\end{eqnarray*}
	Squaring both sides and using the fact that $(x+y)^2\leq 2x^2+2y^2$, we get that
	\begin{eqnarray}\label{eq:wvend}
	\Big\|
	\frac{1}{m}\sum_{i=1}^{m}\hat{\bw}_i-\bv_1
	\Big\|^2~ \leq ~
	2(c')^2\Big(\frac{b^4\log^2(2dm/\fprob)}{\gap^4 n^2}
	+\frac{b^2\log(2dm/\fprob)}{\gap^2 mn}\Big).
	\end{eqnarray}
	This holds with probability at least $1-\fprob-d\exp\left(-\gap^2 n/c b^2\right)$. To simplify things a bit, note that we can assume $d\exp(-\gap^2n/cb^2)\leq \fprob$ without loss of generality, since otherwise the bound in the displayed equation above is at least a constant and therefore trivially true (holds with probability $1$) if we make the constant $c'$ sufficiently large. Therefore, we can argue that \eqref{eq:wvend} (with an appropriate $c'$) holds with probability at least $1-2\fprob$. Absorbing the $2$ factor into the $\fprob$ term, slightly increasing $c'$ appropriately, and simplifying a bit, the result finally follows from the simple observation that
\begin{eqnarray*}
(\v_1^{\top}\w)^2 &=& \frac{1}{2}\left({2 - \Vert{\w-\v_1}\Vert^2}\right) \geq 
\frac{1}{2}\Big({2 - 2\Big\|{\w-\frac{1}{m}\sum_{i=1}^{m}\hat{\bw}_i}\Big\|^2 - 2\Big\|\frac{1}{m}\sum_{i=1}^{m}\hat{\bw}_i-\bv_1\Big\|^2}\Big) \\
&\geq&  1 - 2\Big\|\frac{1}{m}\sum_{i=1}^{m}\hat{\bw}_i-\bv_1\Big\|^2,
\end{eqnarray*}
where the first inequality follows from the triangle inequality and the inequality $(a+b)^2 \leq 2a^2+2b^2$, and the second inequality follows since $\v_1$ is a unit vector, and by definition, $\w$ is the unit vector closest to $\frac{1}{m}\sum_{i=1}^{m}\hat{\bw}_i$.
\end{proof}

\subsection{Lower Bound for Sign Fixing}

We now show that the result of Theorem \ref{thm:ermavg} is tight up to 
poly-logarithmic factors and cannot be improved in general:

\begin{theorem}\label{thm:signFixLB}
For any $\gap \in (0,1)$ and $d>1$, there exist a distribution over vectors in $\reals^d$ (of norm at most a universal constant) with eigengap $\delta$ in the covariance matrix, such that for any number of machines $m$ and for per-machine sample size any $n$ sufficiently larger than $1/\gap^2$, the aggregated vector $\bar{\bv}_1=\frac{1}{m}\sum_{i=1}^{m}\hat{\bv}_1^{(i)}$ (even after sign fixing with the population eigenvector $\v_1$) satisfies
	\[
	\E\left[{1 - \left\langle{\frac{\bar{\v}_1}{\Vert{\bar{\v}_1}\Vert}, 
	\be_1}\right\rangle^2}\right] ~=~ 
	\Omega\left(\frac{1}{\delta^2mn}+ 
	\frac{1}{\gap^4n^2}\right)
	\]
\end{theorem}

The proof is given in the appendix.

\section{A Multi-round Algorithm based on Shift-and-Invert 
Iterations}\label{sec:multiRoundAlgs}

In this section we move on to consider distributed algorithms that perform multiple communication rounds. The main motivation, beyond improving some poly-logarithmic factors in the estimation error, is to obtain a result that does not require the per-machine sample size $n$ to grow with the number of machines $m$, as in the result of Theorem \ref{thm:ermavg}.

Towards this end we consider the use of the Shift-and-Invert meta-algorithm, originally described in  \cite{Garber15,Garber16}, to explicitly solve the centralized ERM objective, i.e., find a unit vector that is an approximate solution to $\max_{\v:\Vert{\v}\Vert =1}\v^{\top}\hat{\X}\v$.

Throughout this section we let $\hat{\lambda}_1 ,\hat{\delta}$ denote the leading eigenvalue and eigengap of $\hat{\X}$, respectively. Also, we assume without loss of generality that $b=1$ (i.e., all data points lie in the unit Euclidean ball).


Since our approach is to approximate the population risk by approximating the empirical risk, we state the following simple lemma for completeness (a proof is given in the appendix).
\begin{lemma}[Risk of approximated-ERM for PCA]\label{lem:approxERM4PCA}
Let $\w$ be a unit vector such that $(\w^{\top}\hat{\v}_1)^2 \geq 1- \epsilon$, for some fixed $\epsilon > 0$, where $\hat{\v}_1$ is the leading eigenvector of $\hat{\X}$. Then it holds that $1 - (\w^{\top}\v_1)^2 \leq 1 - (\w^{\top}\hat{\v}_1)^2 + \sqrt{2\epsilon}$.
\end{lemma}


\subsection{The Shift-and-Invert meta-algorithm} 

The Shift-and-Invert algorithm \cite{Garber15,Garber16} efficiently reduces the 
problem of computing the leading eigenvector of a positive semidefinite matrix 
$\hat{\X}$ to that of approximately-solving a poly-logarithmic number of linear 
systems, i.e., finding approximate minimizers of convex quadratic optimization 
problems of the form
\begin{eqnarray}\label{eq:quadProb}
\min_{\z\in\reals^d}\{F_{\lambda,\w}(\z) :=  \frac{1}{2}\z^{\top}(\lambda\I - \hat{\X})\z - \z^{\top}\w\},
\end{eqnarray}
where $\lambda > \lambda_1(\hat{\X})$ is a shifting parameter. The algorithm is essentially based on applying \textit{power iterations} to a shifted and inverted matrix $(\lambda\I-\hat{\X})^{-1}$, where the shifting parameter $\lambda$ is carefully chosen.
The algorithm that implements this reduction, originally described in \cite{Garber15},  is given below (see Algorithm \ref{alg:convexEV}).

\begin{algorithm}
\caption{\textsc{Shift-and-Invert Power Method}}
\label{alg:convexEV}
\begin{algorithmic}[1]
\STATE Input: estimate $\tilde{\delta}$ for the gap $\hat{\gap}$,  accuracy $\epsilon\in(0,1)$, failure probability $p$
\STATE Set: $m_1 \gets \lceil{8\ln\left({144d/p^2}\right)}\rceil,  m_2 \gets \lceil{\frac{3}{2}\ln\left({\frac{18d}{p^2\epsilon}}\right)}\rceil$
\STATE Set: $\tilde{\epsilon} \gets \min\Big\{\frac{1}{16}\left({\tilde{\delta}/8}\right)^{m_1+1}, \frac{\epsilon}{4}\left({\tilde{\delta}/8}\right)^{m_2+1}\Big\} $
\STATE Set: $\lambda_{(0)} \gets 1+ \tilde{\delta}~, ~~ \hat{\w}_0$ $\gets$ random unit vector, $~ s \gets 0$
\REPEAT
\STATE $s \gets s+1~,~~ \M_s \gets (\lambda_{(s-1)}\I-\hat{\X})$
\FOR{$t=1...m_1$}
\STATE Find an approximate minimizer -  $\hat{\w}_t$  of $F_{\lambda_{(s-1)},\hat{\w}_{t-1}}(\z)$ such that $\Vert{\hat{\w}_t - \M_s^{-1}\hat{\w}_{t-1}}\Vert \leq \tilde{\epsilon}$
\ENDFOR
\STATE $\w_s \gets \hat{\w}_{m_1}/\Vert{\hat{\w}_{m_1}}\Vert$
\STATE Find an approximate minimizer - $\v_s$ of $F_{\lambda_{(s-1)}, \w_s}(\z)$ such that $\Vert{\v_s - \M_s^{-1}\w_s}\Vert \leq \tilde{\epsilon}$
\STATE $\Delta_s \gets \frac{1}{2}\cdot\frac{1}{ \w_s^{\top}\v_s - \tilde{\epsilon}}~, ~~\lambda_{(s)} \gets \lambda_{(s-1)} - \frac{\Delta_s}{2}$
\UNTIL{$\Delta_s \leq \tilde{\delta}$}
\STATE $\lambda_{(f)} \gets \lambda_{(s)}~,~~ \M_f \gets (\lambda_{(f)}\I-\hat{\X})$
\FOR{$t=1...m_2$}
\STATE Find an approximate minimizer - $\hat{\w}_t$ of $F_{\lambda_{(f)}, \hat{\w}_{t-1}}(\z)$
such that $\Vert{\hat{\w}_t - \M_f^{-1}\hat{\w}_{t-1}}\Vert \leq \tilde{\epsilon}$
\ENDFOR
\STATE Return: $\w_f \gets \hat{\w}_{m_2}/\Vert{\hat{\w}_{m_2}}\Vert$
\end{algorithmic}
\end{algorithm}

\begin{lemma}[Efficient reduction of top eigenvector to convex optimization;  
originally Theorem 4.2 in \cite{Garber15}]\label{lem:convexEV}
Suppose that $\hat{\delta} := \lambda_1(\hat{\X}) - \lambda_2(\hat{\X}) > 0$ and suppose that the estimate $\tilde{\gap}$ in Algorithm \ref{alg:convexEV} satisfies $\tilde{\gap}\in[\hat{\gap}/2, 3\hat{\gap}/4]$. Then, with probability at least $1-\fprob$, Algorithm \ref{alg:convexEV} finds a unit vector $\w_f$ such that $(\w_f^{\top}\hat{\v}_1)^2 \geq 1-\epsilon$, and the total number of optimization problems of the form \eqref{eq:quadProb} solved during the run of the algorithm, is upper bounded by $O\left({\ln(d/p)\ln(\hat{\gap}^{-1})+\ln\left({\frac{d}{p\epsilon}}\right)}\right)$.
Moreover, throughout the run of the algorithm it holds that $1+\hat{\delta} \geq \lambda_{(s)} - \hat{\lambda}_1 = \Omega(\hat{\delta})$.
\end{lemma}

\paragraph{Remark:} the purpose of the \textit{repeat-until} loop in Algorithm 
\ref{alg:convexEV} is to efficiently find a shifting parameter $\lambda_{(f)}$ 
such that $c_1\hat{\delta} \leq \lambda_{(f)} - \hat{\lambda}_1 \leq 
c_2\hat{\delta}$ for some universal constants $c_2 > c_1 > 0$. When $n$ 
satisfies $n = \Omega(\delta^{-2}\ln(d/p))$, it follows that we can directly 
find (with high probability) such a shifting parameter, by simply estimating 
$\hat{\lambda}_1,\hat{\delta}$ from the data of a single machine, without any 
communication overhead. Also, in this regime, instead of taking the vector 
$\hat{\w}_0$ to be arbitrary, we can take it to be the leading eigenvector of 
any single machine, since this will already have a constant correlation with 
$\hat{\v}_1$ (with high probability). Thus, for such $n$, the total number of 
optimization problems can be reduced to $O(\ln(p^{-1}\epsilon^{-1}))$.

Algorithm \ref{alg:convexEV} is a meta-algorithm in the sense that the choice 
of solver for the optimization problems $\min{}F_{\lambda,\w}$ is unspecified, 
and any solver will do.
A simple calculation shows that a naive application of either the conjugate gradient method or Nesterov's accelerated gradient method to solve these optimization problems in a distributed manner, i.e., the computation of the gradient vector is distributed across machines, will require overall $\tilde{O}\big({\sqrt{\hat{\lambda}_1/\hat{\gap}}}\big)$  communication rounds, which does not give any improvement over the distributed Lanczos approach, described in Subsection \ref{sec:relatedWork}. However, this can be substantially improved by taking advantage of the fact that the data on all machines is sampled i.i.d. from the same distribution. In particular, we present below an approach based on applying a pre-conditioner to the optimization Problem \eqref{eq:quadProb}, in the spirit of the one described in \cite{Disco}.

\subsection{Faster Distributed Approximation of Linear Systems via Local Preconditioning}\label{sec:preconditionLinSys}

Let $\M = \lambda\I-\hat{\X}$, for some shift parameter $\lambda > \hat{\lambda}_1$, and define the pre-conditioning matrix $\C = (\lambda+\mu)\I-\hat{\X}_1$, where $\mu$ is required so $\C$ is invertible. Consider now solving the following modified quadratic problem:
\begin{eqnarray}\label{eq:quadProbCond}
\tilde{F}_{\lambda,\w}(\y) :=  \frac{1}{2}\y^{\top}\C^{-1/2}\M\C^{-1/2}\y - \y^{\top}\C^{-1/2}\w.
\end{eqnarray}

Note that if $\y^*$ is the optimal solution to Problem \eqref{eq:quadProbCond}, i.e., 
\begin{eqnarray*} 
\y^* = \C^{1/2}\M^{-1}\C^{1/2}\C^{-1/2}\w = \C^{1/2}\M^{-1}\w,
\end{eqnarray*}
then $\z^* := \C^{-1/2}\y^*$ is the optimal solution to Problem \eqref{eq:quadProb}.

The idea behind choosing $\C$ this way is very intuitive. Ideally we could have 
chosen $\C=\M$, making the condition number of $\tilde{F}_{\lambda,\w}$ equal to
$\kappa(\tilde{F}_{\lambda,\w}) = 1$, which is the best we can hope for. The 
problem of course is that this requires us to explicitly compute $\M^{-1/2}$, 
which is more challenging then just computing the leading eigenvector of 
$\hat{\X}$. The next best thing is thus to choose $\C$ based only on the data 
available on any single machine, which allows computing $\C^{-1/2}$ 
without additional communication overhead, and leads to the choice described 
above. The following lemma, rephrased from \cite{Disco}, quantifies exactly how 
such a choice of $\C$ helps in improving the condition number of the new 
optimization problem, Problem \eqref{eq:quadProbCond}. The proof is given in 
the appendix. 

\begin{lemma}\label{lem:precondLS}
Suppose that $\mu \geq \Vert{\hat{\X} - \hat{\X}_1}\Vert$. Then, $\tilde{F}_{\lambda,\w}(\y)$ is $1$-smooth and $\Big({\frac{\lambda-\hat{\lambda}_1}{(\lambda-\hat{\lambda}_1) + 2\mu}}\Big)$-strongly convex. In particular, 
The condition number\footnote{defined as the smoothness parameter  divided by the strong-convexity parameter.} $\kappa\left({\tilde{F}_{\lambda,\w}}\right)$ satisfies
\begin{eqnarray*}
\kappa\left({\tilde{F}_{\lambda,\w}}\right) \leq 1 + \frac{2\mu}{\lambda-\lambda_1(\hat{\X})}.
\end{eqnarray*}
Moreover, fixing $\tilde{\y}\in\reals^d$, if we let $\tilde{\z} := \C^{-1/2}\tilde{\y}$, then it holds that $$\Vert{\tilde{\z} - \M^{-1}\w}\Vert \leq (\lambda-\hat{\lambda}_1)^{-1/2}\Vert{\tilde{\y} - \C^{1/2}\M^{-1}\w}\Vert.$$
Finally, for any $p\in(0,1)$, if we set $\mu = 4\sqrt{\ln(d/p)/n}$, then the above holds with probability at least $1-p$, 
where this probability depends only on the randomness in $\hat{\X}_1$.
\end{lemma}

\subsubsection{Solving the pre-conditioned linear systems}

We now discuss the application of gradient-based algorithms for finding an approximate minimizer of the pre-conditioned problem, Problem \eqref{eq:quadProbCond}, in our distributed setting. Towards this end we require a \textit{distributed} implementation for the \textit{first-order oracle} of $\tilde{F}_{\lambda,\w}(\y)$ (i.e., computation of the value and gradient vector at a queried point).


A straight-forward implementation of the first-order oracle in our distributed setting is given in Algorithm \ref{alg:distFOO}. 

\begin{algorithm}
\caption{Distributed First-Order Oracle for $\tilde{F}_{\lambda,\w}(\y)$}
\label{alg:distFOO}
\begin{algorithmic}[1]
\STATE Input: shift parameter $\lambda > 0$, regularization parameter $\mu > 0$, vector $\w\in\reals^d$, query vector $\y\in\reals^d$
\STATE send $\tilde{\y} := \C^{-1/2}\y$ to machines $\{2,\ldots,m\}$ for $\C := (\lambda+\mu)\I -\hat{\X}_1$ \COMMENT{executed on machine 1}
\FOR{$i=1...m$}
\STATE send $\tilde{\nabla}_i := \hat{\X}_i\tilde{\y}$ to machine 1 \COMMENT{executed on each machine $i$}
\ENDFOR
\STATE aggregate $\tilde{\nabla} := \frac{1}{m}\sum_{i=1}^m\tilde{\nabla}_i$ \COMMENT{executed on machine 1}
\STATE compute $\tilde{F}_{\lambda,\w}(\y) = \frac{1}{2}(\lambda\y^{\top}\C^{-1}\y - \y^{\top}\C^{-1/2}\tilde{\nabla}) - \y^{\top}\C^{-1/2}\w$ \COMMENT{executed on machine 1}
\STATE compute $\nabla\tilde{F}_{\lambda,\w}(\y) = \lambda\C^{-1}\y - \C^{-1/2}\tilde{\nabla} - \C^{-1/2}\w$ \COMMENT{executed on machine 1}
\STATE return: $(\tilde{F}_{\lambda,\w}(\y), \nabla\tilde{F}_{\lambda,\w}(\y))$
\end{algorithmic}
\end{algorithm}

We have the following lemma, the proof of which is deferred to the appendix.
\begin{lemma}\label{lem:distPrecondLS}
Fix some $\lambda > \lambda_1(\hat{\X})$ and $\w\in\reals^d$, and let $1\geq \mu >0$ be as in Lemma \ref{lem:precondLS}. Fix $\epsilon > 0$. 
Consider the following two-step algorithm:
\begin{enumerate}
\item
Apply either the conjugate gradient method or Nesterov's accelerated method with the distributed first-order oracle described in Algorithm \ref{alg:distFOO} to find $\tilde{\y}\in\reals^d$ such that $\tilde{F}_{\lambda,\w}(\tilde{\y}) - \min_{\y\in\reals^d}\tilde{F}_{\lambda,\w}(\y) \leq \epsilon'$
\item
Return $\tilde{\z} = \C^{-1/2}\tilde{\y}$.
\end{enumerate}
Then, for $\epsilon' = \frac{\epsilon}{2}\left({1+ \frac{2\mu}{\lambda-\hat{\lambda}_1}}\right)^{-1}(\lambda-\hat{\lambda}_1)$ it holds that $\Vert{\tilde{\z} - (\lambda\I-\hat{\X}_1)^{-1}\w}\Vert \leq \epsilon$, and the total number distributed matrix-vector products with the empirical covariance matrix $\hat{\X}$ required to compute $\tilde{\z}$ is upper-bounded by
\begin{eqnarray*}
O\left({\sqrt{1 + 2\mu(\lambda-\hat{\lambda}_1)^{-1}}\ln\Big({\Big({1 + \frac{2\mu}{\lambda-\hat{\lambda}_1}}\Big)\Vert{\w}\Vert/[(\lambda-\hat{\lambda}_1)\epsilon]}\Big)}\right).
\end{eqnarray*}
\end{lemma}

\subsection{Putting it all together}

We now state our main result for this section, which is a simple 
consequence of the previous lemmas. The full proof is given in the appendix.

\begin{theorem}\label{thm:shiftNinvert:main}
Fix $\epsilon\in(0,1)$ and $p\in(0,1)$. Suppose that $mn = 
\Omega(\delta^{-2}\ln(d/p))$. Set $\mu =  4\sqrt{\ln(3d/p)/n}$. Applying the 
Shift-and-Invert algorithm, Algorithm \ref{alg:convexEV}, with the parameters 
$\epsilon,p/3$, and applying the algorithm in Lemma \ref{lem:distPrecondLS} 
with the parameter $\mu$, to approximately solve the linear systems, yields 
with probability at least $1-p$ a unit vector $\w_f$ such that 
$(\w_f^{\top}\hat{\v}_1)^2 \geq 1-\epsilon$, after executing at most
\begin{align*}
O\left(\sqrt{\frac{\sqrt{\ln(d/p)}}{\delta\sqrt{n}}}\left[\ln\left(\frac{d}{p\epsilon^2
}\right)\ln\left(\frac{\sqrt{\ln(d/p)}}{\delta^2\sqrt{n}}\right)+\ln^2\left(\frac{d}{p\epsilon^2}\right)
\ln\left(\frac{1}{\delta}\right)\right]\right)
~=~ \tilde{O}\left(\sqrt{\frac{1}{\delta\sqrt{n}}}\right)
\end{align*}
distributed matrix-vector products with the empirical covariance matrix 
$\hat{\X}$. 
\end{theorem}

\paragraph{Remark:} Our approach of using Shift-and-Invert with the preconditioning technique for linear systems is applicable in a much more general setting. Namely, all that is required for the method to obtain accelerated rates over standard algorithms, is (1) a non-zero gap in the aggregated empirical matrix, i.e., $\delta(\hat{\X}) > 0$, and (2) that the distance $\Vert{\hat{\X} - \hat{\X}_1}\Vert$ admits a non-trivial upper-bound.

\section{Experiments}

To validate some of our theoretical findings we conducted experiments with single-round algorithms on synthetic data. 
We generated synthetic datasets using two distributions. For both distributions we used the covariance matrix $\X = \U\mathbf{\Sigma}\U^{\top}$ with $\U$ being a random $d\times d$ orthonormal matrix and $\mathbf{\Sigma}$ is diagonal satisfying:
$\mathbf{\Sigma}(1,1) = 1, ~ \mathbf{\Sigma}(2,2) = 0.8, ~ \forall j\geq 3: ~ 
\mathbf{\Sigma}(j,j) = 0.9\cdot\mathbf{\Sigma}(j-1,j-1)$,
i.e., $\delta = 0.2$. 
One dataset was generated according to the normal distributions 
$\mathcal{N}(0,\X)$, and for the second datasets we generated samples by taking 
$\x = \sqrt{3/2}\X^{1/2}\y$ where $\y\sim{}U[-1,1]$. 
In both cases we set $d = 300$. 

Beyond the single-round algorithms that are based on aggregating the individual 
ERM solutions described so far, we propose an additional natural aggregation 
approach, based on aggregating the individual projection matrices. More 
concretely, letting $\{\hat{\v}_1^{(i)}\}_{i=1}^m$ denote the leading 
eigenvectors of the individual machines, let $\bar{\P}_1 := 
\frac{1}{m}\sum_{i=1}^m\hat{\v}_1^{(i)}\hat{\v}_1^{(i)\top}$. We then take the 
final estimate $\w$ to be the leading eigenvector of the aggregated matrix 
$\bar{\P}_1$. Note that as with the sign-fixing based aggregation, this 
approach also resolves the sign-ambiguity in the estimates produced by the 
different machines, which circumvents the lower bound result of Theorem 
\ref{thm:simpleAvgFail}. 

For both datasets we fixed the number of machines to $m = 25$.
We tested the estimation error (i.e., the value $1 - (\w^{\top}\v_1)^2$ where 
$\v_1$ is the leading eigenvector of $\X$ and $\w$ is the estimator) of five 
benchmarks vs. the per-machine sample size $n$: the centralized solution 
$\hat{\v}_1$, the average of the individual (unbiased) ERM solutions 
(normalized to unit norm),
  the average of ERM solutions with sign-fixing, and the leading eigenvector of 
  the averaged projection matrix. We also plotted the average loss of the 
  individual ERM solutions. Results are averaged over 400 independent runs.

The results appear in Figure \ref{fig:expresults}. 
It is observable that the results for both distributions are very similar.
We can see that, as our lower bound in Theorem \ref{thm:simpleAvgFail} 
suggests, simply averaging and normalizing the individual ERM solutions has 
significantly worse performance than the centralized ERM solution. 
Perhaps surprisingly, the performance of this estimator is even worse than the 
average error of an estimate computed using only a single machine.
We see that both aggregation methods that are based on correlating the 
individual ERM solutions, namely the sign-fixing-based estimator, and the 
proposed averaging-of-projections heuristic, are asymptotically consistent with 
the centralized ERM. In particular, the averaging-of-projections scheme, at 
least empirically, significantly outperforms the sign-fixing approach, which 
justifies further theoretical investigation of this heuristic. For the sign 
fixing approach, we can see that as suggested by our bounds, the 
estimator is not consistent with the centralized ERM solution for small values 
of $n$.

\begin{figure}[h]
    \centering
    \begin{subfigure}[b]{0.485\textwidth}
        \includegraphics[width=\textwidth]{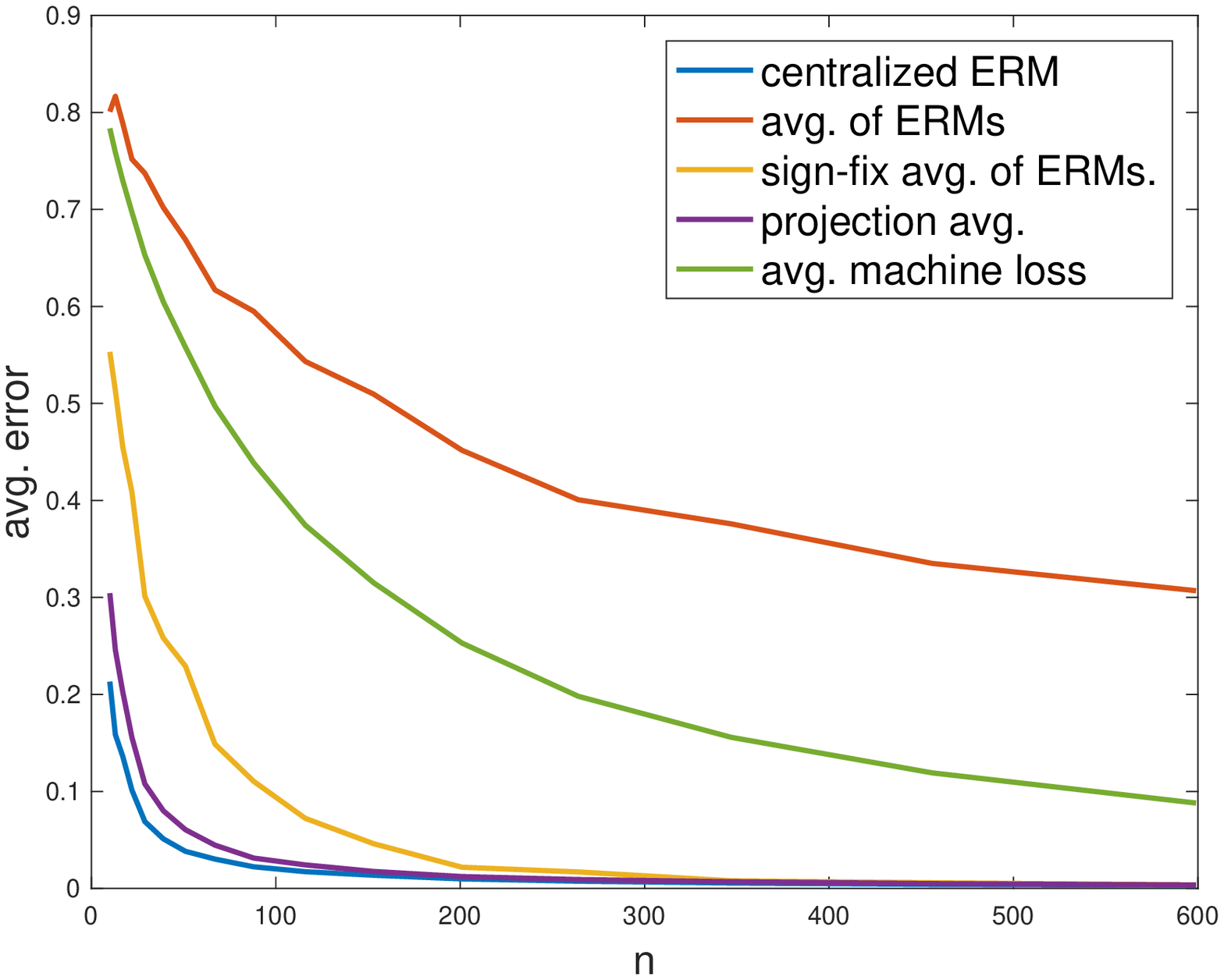}
    \end{subfigure}
          \begin{subfigure}[b]{0.485\textwidth}
        \includegraphics[width=\textwidth]{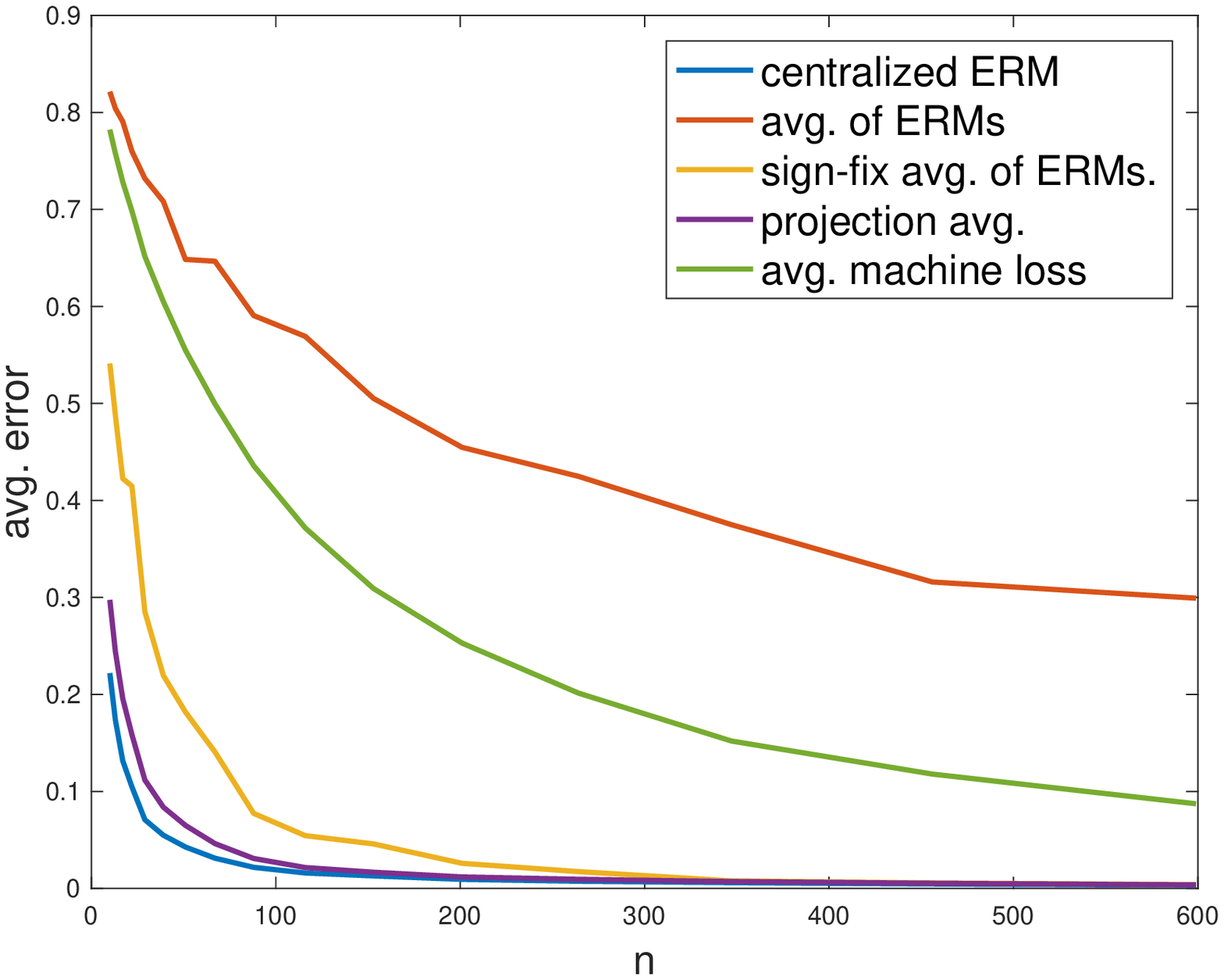}
    \end{subfigure}
    \caption{Estimation error vs. the per-machine sample size $n$ for a normal distribution (left) and uniform sampling-based distribution (right).}\label{fig:expresults}
\end{figure}

\bibliography{bib}
\bibliographystyle{plain}

\appendix

\section{Proofs Omitted from Section \ref{sec:singleRoundAlgs}}

\subsection{Proof of Theorem \ref{thm:simpleAvgFail}}
\begin{proof}
Consider the following distribution over $\reals^2$.
\begin{eqnarray*}
\x = \be_1 + \left( \begin{array}{c}
\epsilon_1  \\
\epsilon_2 \end{array} \right)
, \quad \epsilon_1,\epsilon_2 \sim U\{-1,+1\},
\end{eqnarray*}
where $\be_1$ is the first standard basis vector in $\reals^2$.

The population covariance matrix and the empirical covariance matrix of a sample of size $n$ are clearly given by
\begin{eqnarray*}
\X =
\left( \begin{array}{cc}
2 & 0  \\
0 & 1 \end{array} \right), \qquad 
\hat{\X}_{(n)} =
\left( \begin{array}{cc}
2 & y_n  \\
y_n & 1 \end{array} \right),
\end{eqnarray*}
where $y_n$ is a random variable which is the average of $n$ $U\{-1,+1\}$ random variables. 
By elementary calculations we have that the leading eigenvector of $\hat{\X}_{(n)}$ is given by
\begin{eqnarray*}
\hat{\v}_1 = \sigma\cdot C(y_n)\cdot \left({1, ~ \frac{2y_n}{1+\sqrt{1+4y_n^2}}}\right),
\end{eqnarray*}
where 
\begin{eqnarray*}
C(y_n) := \left({1 + \left({\frac{2y_n}{1+\sqrt{1+4y_n^2}}}\right)^2}\right)^{-1/2}
\end{eqnarray*}
is the normalization factor that guarantees that $\hat{\v}_1$ is a unit vector. In particular, it holds that $1/\sqrt{2} \leq C(y_n) \leq 1$.
The random variable $\sigma\sim U\{-1,+1\}$ is independent of $y_n$ and determines the sign of $\hat{\v}_1$, which follows from our assumption that $\hat{\v}_1$ is generated by unbiased ERM.

Consider now the average of $m$ such unit vectors $\hat{\v}_1^{(1)}..\hat{\v}_1^{(m)}$ given by $\bar{\v} = \frac{1}{m}\sum_{i=1}^m\hat{\v}_1^{(i)}$
and the normalized estimate $\bar{\v}_1/\Vert{\bar{\v}_1}\Vert$, and recall that the leading eigenvector of the population covariance matrix is $\be_1$. It holds that
\begin{eqnarray}\label{eq:thm:simpleAvgFail:1}
\langle{\frac{\bar{\v}_1}{\Vert{\bar{\v}_1}\Vert}, \be_1}\rangle^2 = \frac{\bar{\v}_1(1)^2}{\bar{\v}_1(1)^2 + \bar{\v}_1(2)^2} =1- \frac{\bar{\v}_1(2)^2}{\bar{\v}_1(1)^2 + \bar{\v}_1(2)^2}.
\end{eqnarray}

Towards upper-bounding the RHS of \eqref{eq:thm:simpleAvgFail:1} in expectation, the main step is to lower bound the random variable $\vert{\bar{\v}_1(2)}\vert$ using Chebyshev's inequality. 

It holds that
\begin{eqnarray}\label{eq:thm:simpleAvgFail:2}
\E[\left\vert{\bar{\v}_1(2)}\right\vert] &=& \E\left[{\left\vert{\frac{1}{m}\hat{\v}_1^{(i)}(2)}\right\vert}\right] 
= \E\left[{\left\vert{\frac{1}{m}\sum_{i=1}^m\sigma^{(i)}\frac{2C(y^{(i)}_n)y^{(i)}_n}{1+\sqrt{1+4y_n^{(i)2}}}}\right\vert}\right] \nonumber\\
&\underset{(a)}{=}& \E\left[{\left\vert{\frac{1}{m}\sum_{i=1}^m\sigma^{(i)}\frac{2C(y^{(i)}_n)\vert{y^{(i)}_n}\vert}{1+\sqrt{1+4y_n^{(i)2}}}}\right\vert}\right]\nonumber \\
&=& \E_{\{\sigma^{(i)}\}}\left[{\E_{\{y_n^{(i)}\}}\left[{\left\vert{\frac{1}{m}\sum_{i=1}^m\sigma^{(i)}\frac{2C(y^{(i)}_n)\vert{y^{(i)}_n}\vert}{1+\sqrt{1+4y_n^{(i)2}}}}\right\vert ~ | ~ \{\sigma^{(i)}\}}\right]}\right] \nonumber\\
& \underset{(b)}{\geq} & \E_{\{\sigma^{(i)}\}}\left[{\left\vert{\E_{\{y_n^{(i)}\}}\left[{\frac{1}{m}\sum_{i=1}^m\sigma^{(i)}\frac{2C(y^{(i)}_n)\vert{y^{(i)}_n}\vert}{1+\sqrt{1+4y_n^{(i)2}}} ~ | ~ \{\sigma^{(i)}\}}\right]}\right\vert}\right] \nonumber\\
&\underset{(c)}{=} & \E_{\{\sigma^{(i)}\}}\left[{\left\vert{\frac{1}{m}\sum_{i=1}^m\sigma^{(i)}}\right\vert}\right]\cdot\E_{y_n}\left[{\frac{2C(y_n)\vert{y_n}\vert}{1+\sqrt{1+4y_n^2}}}\right] 
 \underset{(d)}{=} \Theta\left({\frac{1}{\sqrt{mn}}}\right),
\end{eqnarray}
where (a) follows since $\sigma^{(i)}y_n^{(i)} \sim \sigma^{(i)}\vert{y_n^{(i)}}\vert$ and $C(y_n^{(i)})/(1+\sqrt{1+4y_n^{(i)2}})$ depends only on $\vert{y_n^{(i)}}\vert$, (b) follows from the triangle inequality, and (c) follows since $\{\sigma^{(i)}\}_{i\in[m]}$ and $\{y_n^{(i)}\}_{i\in[m]}$ are independent random variables. Finally, it is easy to verify that (d) follows since $\sum_{i=1}^m\sigma^{(i)}/m$ is the average of $m$ $U\{-1,+1\}$ random variables and hence its expected absolute value is $\Theta(1/\sqrt{m})$. Similarly the expected absolute value of $y_n$ is $\Theta(1/\sqrt{n})$ and $C(y_n) / (1+\sqrt{1+4y_n^2})$ is lower bounded by a positive constant.

Also, observe that
\begin{eqnarray}\label{eq:thm:simpleAvgFail:3}
\E[\bar{\v}_1(2)^2] &=& \E\left[{\left({\frac{1}{m}\hat{\v}_1^{(i)}(2)}\right)^2}\right] = \frac{1}{m}\E[\hat{\v}_1(2)^2] = \frac{1}{m}\E\left[{\left({\frac{2C(y_n)y_n}{1+\sqrt{1+4y_n^2}}}\right)^2}\right]  \nonumber\\
&\geq & \frac{1}{2m}\E[y_n^2] = \Theta\left({\frac{1}{mn}}\right),
\end{eqnarray}
where the inequality follows since $\vert{y_n}\vert \leq 1$ and $1/\sqrt{2}\leq C(y_n)\leq 1$.

Combining Eq. \eqref{eq:thm:simpleAvgFail:2} and Eq. \eqref{eq:thm:simpleAvgFail:3}, we have by an application of Chebyshev's inequality to the random variable $\vert{\bar{\v}_1(2)}\vert$ that there exists universal constants $c_1>0$ such that
\begin{eqnarray}\label{eq:thm:simpleAvgFail:4}
\Pr\left({\vert{\bar{\v}_1(2)}\vert \leq \frac{1}{c_1\sqrt{mn}}}\right) \leq \frac{1}{4} .
\end{eqnarray}

Also, it is easy to verify that
\begin{eqnarray*}
\E[\bar{\v}_1(1)^2] = O(1/m), \qquad \E[\bar{\v}_1(2)^2] = O(1/m) .
\end{eqnarray*}

Thus, by a simple application of Markov's inequality we have that there exists a universal constant $c_2>0$ such that
\begin{eqnarray}\label{eq:thm:simpleAvgFail:5}
\Pr\left({\max\{\bar{\v}_1(1)^2, ~ \bar{\v}_1(2)^2\} \geq \frac{1}{c_3m}}\right) \leq \frac{1}{4}.
\end{eqnarray}

Using Eq. \eqref{eq:thm:simpleAvgFail:1}, \eqref{eq:thm:simpleAvgFail:4} and \eqref{eq:thm:simpleAvgFail:5} we finally have that
\begin{eqnarray*}
\E\left[{\langle{\frac{\bar{\v}_1}{\Vert{\bar{\v}_1}\Vert}, \be_1}\rangle^2}\right] = 1- \E\left[{\frac{\bar{\v}_1(2)^2}{\bar{\v}_1(1)^2 + \bar{\v}_1(2)^2}}\right]
= 1- \Omega\left({\frac{1}{n}}\right).
\end{eqnarray*}


 \end{proof}

\subsection{Proof of Theorem \ref{thm:signFixLB}}

The proof is a combination of the following two lemmas, each proves one of the lower bounds. We first state the two lemmas and then prove them.

\begin{lemma}\label{lem:signFixLB:1}
For any $\gap \in (0,1)$ and $d>1$, there exist a distribution over vectors in $\reals^d$ (of norm at most $2$) such that the covariance matrix has eigengap $\gap$, and for any number of machines $m$ and per-machine sample size $n$, the aggregated vector $\bar{\bv}_1=\frac{1}{m}\sum_{i=1}^{m}\hat{\bv}_1^{(i)}$ (even after sign fixing) satisfies
	\[
	\E\left[{1 - \langle{\frac{\bar{\v}_1}{\Vert{\bar{\v}_1}\Vert}, \be_1}\rangle^2}\right] ~=~ \Omega\left({\min\left\{\frac{1}{m},\frac{1}{\gap^2mn}\right\}}\right).
	\]
\end{lemma}


\begin{lemma}\label{lem:signFixLB:2}
For any $\gap \in (0,1)$ and $d>1$, there exist a distribution over vectors in $\reals^d$ (of norm at most $2$) with eigengap $\delta$ in the covariance matrix, such that for any number of machines $m$ and for per-machine sample size any $n$ sufficiently larger than $1/\gap^2$, the aggregated vector $\bar{\bv}_1=\frac{1}{m}\sum_{i=1}^{m}\hat{\bv}_1^{(i)}$ (even after sign fixing with the population eigenvector $\v_1$) satisfies
	\[
	\E\left[{1 - \langle{\frac{\bar{\v}_1}{\Vert{\bar{\v}_1}\Vert}, \be_1}\rangle^2}\right] ~=~ \Omega\left({\frac{1}{\gap^4n^2}}\right).
	\]
\end{lemma}

\begin{proof}[proof of Lemma \ref{lem:signFixLB:1}]
We will prove the result for $d=2$ (i.e. a distribution in $\reals^2$). This is without loss of generality, since we can always embed the distribution below in $\reals^d$ for any $d>2$ (say, by having all coordinates other than the first two identically zero).

Consider the distribution defined by the random vector 
$\bx = \sqrt{1+\gap}\be_1+\sigma\be_2$, where $\sigma$ is uniformly distributed on $\{-1,+1\}$, and $\be_1=(1,0),\be_2=(0,1)$ are the standard basis vectors. 
Clearly, the population covariance matrix is
\[
\X := \E[\bx\bx^\top] = \left(\begin{array}{cc}
1+\gap & 0\\ 0 & 1\end{array}\right),
\]
with a leading eigenvector $(1,0)$. Let us now consider the distribution of the output of a machine $i$. Given $n$ samples, the empirical covariance matrix is
\[
\hat{\X}_{(n)} = \left(\begin{array}{cc}
1+\gap & y_n\\ y_n & 1\end{array}\right)~~,~~ y_n:= \sqrt{1+\gap}\cdot\frac{1}{n}\sum_{i=1}^{n}\epsilon_i,
\]

where $\epsilon_i$ are i.i.d. and uniformly distributed on $\{-1,+1\}$. Using a 
standard formula for the leading eigenvector of a $2\times 2$ matrix 
\cite{eigs2x2}, we have that the leading eigenvector (and hence the output of 
any machine $i$) is of the form
\begin{equation}\label{eq:eigvec}
\hat{\bv}_1 = \frac{1}{\norm{\hat{\bu}}}\hat{\bu}~~~\text{where}~~~ \hat{\bu}:= \left(\frac{\gap}{2}+\sqrt{\frac{\gap^2}{4}+y_n^2}~,~y_n\right).
\end{equation}
Note that with this formula, the leading eigenvector is always closer to $(1,0)$ than $(-1,0)$, and converges to $(1,0)$ as $n\rightarrow \infty$. Thus, we can view the random variable $\hat{\bv}^{(i)}$ as the output of any machine $i$, given $n$ samples and after fixing the sign. 

Consider now the average of $m$ such vectors given by $\bar{\v} = \frac{1}{m}\sum_{i=1}^m\hat{\v}_1^{(i)}$.
Using \eqref{eq:eigvec}, we have that
\begin{eqnarray}\label{eq:signfixLB:1}
\E[\bar{\v}_1(2)^2] &=& \E\left[\left(\frac{1}{m}\sum_{i=1}^{m}\hat{\v}^{(i)}_{2}\right)^2\right] = \frac{1}{m^2}\sum_{i=1}^{m}\E[(\hat{\v}^{(i)}_{2})^2] = \frac{1}{m}\E[(\hat{\v}(2))^2] \nonumber \\
&=&\frac{1}{m}\E\left[\frac{y_n^2}{\frac{\gap^2}{2}+2y_n^2+\gap\sqrt{\frac{\gap^2}{4}+y_n^2}}\right].
\end{eqnarray}

By definition of $y_n$ and recalling that $\delta\in[0,1]$, we have that there exist universal constants $c_1,c_2 >0$ such that with constant probability it holds that $c_1/n \geq y_n^2 \geq c_2/n$. Using this fact and considering the two cases $1/n \geq \delta^2$ and $1/n < \delta^2$ in the RHS of Eq. \eqref{eq:signfixLB:1} separately, we can see that
\begin{eqnarray}\label{eq:signfixLB:2}
\E[\bar{\v}_1(2)^2] = \Omega\left({\frac{1}{m}\min\{1, \frac{1}{\delta^2n}\}}\right).
\end{eqnarray}

Using Eq. \eqref{eq:signfixLB:2} we have that
\begin{eqnarray*}
\E\left[{\langle{\frac{\bar{\v}_1}{\Vert{\bar{\v}_1}\Vert}, ~ \be_1}\rangle^2}\right] &=& \E\left[{\frac{\bar{\v}_1(1)^2}{\bar{\v}_1(1)^2+\bar{\v}_1(2)^2}}\right]
= 1 - \E\left[{\frac{\bar{\v}_1(2)^2}{\bar{\v}_1(1)^2+\bar{\v}_1(2)^2}}\right] \\
& \leq &1 - \E\left[{\bar{\v}_1(2)^2}\right] = 1 - \Omega\left({\min\left\{\frac{1}{m},\frac{1}{\gap^2mn}\right\}}\right),
\end{eqnarray*}
where the inequality follows since $\Vert{\bar{\v}_1}\Vert \leq 1$. 

\end{proof}

\begin{proof}[proof of Lemma \ref{lem:signFixLB:2}]
As in Lemma \ref{lem:signFixLB:1}, we prove the result for $d=2$, however, using a different construction. 
Consider the defined by the random vector
\[
\bx = 
\sqrt{1+\gap}\cdot\be_1+\xi\cdot \be_2,
\]
where 
$\xi$ is an independent random variable defined as:
\[ \xi = \left\{ \begin{array}{ll}
         \sqrt{2} & \mbox{w.p. $1/3$}\\
        -1/\sqrt{2} & \mbox{w.p. $2/3$}\end{array} \right. \]

It is easy to verify that $\E[\xi]=0$, $\E[\xi^2]=1$, $\E[\xi^3]  = 1/\sqrt{2}$. As we shall see, choosing $\xi$ to be asymmetric (as opposed to $\epsilon$ in the proof of Lemma \ref{lem:signFixLB:2}) will be key to our construction.
Clearly, the population covariance and the empirical covariance of a sample of size $n$ are given by
we 
have
\[
\X = \E[\bx\bx^\top] = \left(\begin{array}{cc}
1+\gap & 0\\ 0 & 1\end{array}\right) ~,
\qquad
\hat{\X}_{(n)} = \left(\begin{array}{cc}
1+\gap & y_n\\ y_n & z_n\end{array}\right),
\]
where
\[
 y_n:= 
\sqrt{1+\gap}\cdot\frac{1}{n}\sum_{i=1}^{n}\xi_i~, \qquad z_n:= 
\frac{1}{n}\sum_{i=1}^{n}\xi_i^2,
\]
with $\xi_1,\ldots,\xi_n$ being i.i.d. copies of the random variable $\xi$.

Clearly the leading eigenvector of $\X$ is $\be_1 = (1,0)$. 
Consider now $\hat{\v}_1^{(1)},\ldots,\hat{\v}_1^{(m)}$ to be the leading eigenvectors of $m$ i.i.d. empirical covariance matrices of $n$ samples, $\hat{\X}_{(n)}^{(1)},\ldots,\hat{\X}_{(n)}^{m)}$, and let $\bar{\v}_1$ denote their average after sign-fixings according to the leading eigenvector of the population covariance $\be_1$. In the following, we let $\hat{v}_j^i$ denote the $j$th coordinate in the eigenvector $\hat{\v}_1^{(i)}$.

It holds that
\begin{eqnarray}\label{eq:signFixLB2:1}
\E\left[{\langle{\frac{\bar{\v}_1}{\Vert{\bar{\v}_1}\Vert}, ~ \be_1}\rangle^2}\right] &=& \E\left[{\frac{\bar{\v}_1(1)^2}{\bar{\v}_1(1)^2+\bar{\v}_1(2)^2}}\right]
= 1 - \E\left[{\frac{\bar{\v}_1(2)^2}{\bar{\v}_1(1)^2+\bar{\v}_1(2)^2}}\right] \nonumber \\
& \leq &1 - \E\left[{\bar{\v}_1(2)^2}\right] = 1 - \E\left[\left(\frac{1}{m}\sum_{i=1}^{m}\sign(\hat{v}^i_1)\hat{v}^i_2\right)^2\right] \nonumber\\
&\leq& 1 - \left(\frac{1}{m}\sum_{i=1}^{m}\E\left[\sign(\hat{v}^i_1)\hat{v}^i_2\right]\right)^2 \nonumber \\
&=&1- \left(\E[\sign(\hat{v}^1_1)\hat{v}^1_2]\right)^2,
\end{eqnarray}
where the first inequality follows since $\Vert{\bar{\v}_1}\Vert \leq 1$, the second inequality follows from Jensen's inequality, and the last equality follows from the fact that $\hat{\v}_1^{(1)},\ldots,\hat{\v}_1^{(m)}$
are i.i.d. random variables. From this chain of inequalities, it follows that it is enough to 
lower bound $\left(\E[\sign(\hat{v}^1_1)\hat{v}^1_2]\right)^2$, where 
$\hat{\bv}^1$ is the leading eigenvector computed by machine $1$.

Let us now consider the distribution of the leading eigenvector of the empirical covariance matrix $\hat{\X}_{(n)}$.  Using a standard formula 
for the leading eigenvector of a $2\times 2$ matrix \cite{eigs2x2}, we have 
that 
this leading eigenvector $\hat{\bv}_1$ is proportional to
\begin{equation}
\left(\frac{\gap+1-z_n}{2}+\sqrt{\left(\frac{\gap+1-z_n}{2}\right)^2+y_n^2}~,~y_n\right)
\end{equation}
Assume for now that $z_n\leq 1+c\gap$ for some positive constant $c$ to be fixed later (note this happens with arbitrarily 
high probability as $n\rightarrow \infty$, as $z_n$ converges to $1$ in 
probability). In that case, the sign of the first coordinate in the formula above is  positive, and has the same sign as the first coordinate of the leading eigenvector $\bv_1=(1,0)$. Moreover, we know 
that $\hat{\bv}_1^{(1)}$ must have unit norm, from which follows that
\begin{equation}\label{eq:eigvec2}
\sign(\hat{v}^1_1)\cdot\hat{\bv}_1^{(1)} ~=~
\frac{\left(\frac{\gap+1-z_n}{2}+\sqrt{\left(\frac{\gap+1-z_n}{2}\right)^2+y_n^2}
~,~y_n\right)}
{\sqrt{y_n^2+\left(\frac{\gap+1-z_n}{2}+\sqrt{\left(\frac{\gap+1-z_n}{2}\right)^2
+y_n^2}\right)^2}}.
\end{equation}
In particular, letting $r_n = 1-z_n$, we have that if $r_n\geq -c\gap$, then
\begin{align}\label{eq:signFixLB2:2}
\sign(\hat{v}^1_1)\cdot \hat{v}^1_2 &~=~
\frac{y_n}{\sqrt{y_n^2+\left(\frac{\gap+r_n}{2}+\sqrt{\left(\frac{\gap+r_n}{2}\right)^2
+y_n^2}\right)^2}} \nonumber \\
&~=~
\frac{y_n}{\sqrt{y_n^2+\left(\frac{\gap+r_n}{2}\right)^2
\left(1+\sqrt{1
+\left(\frac{2y_n}{\gap+r_n}\right)^2}\right)^2}}.
\end{align}
Towards using Eq. \eqref{eq:signFixLB2:1} to derive the lower bound, the main step is to bound the expectation of the RHS of Eq.\eqref{eq:signFixLB2:2} away from zero. To get an intuition why this is possible, observe that when $n \rightarrow \infty$ (in particular, when it is significantly larger than $1/\delta^2$), it holds that 
\begin{eqnarray*}
\textrm{RHS of \eqref{eq:signFixLB2:2}} ~\approx ~ \frac{y_n}{\sqrt{y_n^2 + \Theta(\delta^2)}},
\end{eqnarray*}
since in this regime, with high probability, $r_n << \delta$ and $y_n << 1$.
Now comes to play our choice of $\xi$ to be an asymmetric random variable. If, just for sake of intuition, we set $n=1$, it is easy to verify that despite the fact that $\E[y_n] =0$, it holds that
\begin{eqnarray*}
\E\left[{\frac{y_n}{\sqrt{y_n^2 + \Theta(\delta^2)}}}\right] = \E\left[{\frac{\xi}{\sqrt{\xi^2 + \Theta(\delta^2)}}}\right] < 0.
\end{eqnarray*}
Note in particular that taking $\xi$ to be uniformly distributed on $\{-1,+1\}$, as in Lemma \ref{lem:signFixLB:1}, will still give zero expectation, and hence will not work.
We now formalize this intuition. We will use a Taylor expansion of the formula above, in order to bound its expectation (over $y_n,r_n$), from which a lower bound on $\left(\E\left[\sign(\hat{v}^1_1)\cdot \hat{v}^1_2\right]\right)^2$ would follow. To that end, define the function
\[
g(t) = \frac{ty_n}{\sqrt{(ty_n)^2+\left(\frac{\gap+t r_n}{2}\right)^2
\left(1+\sqrt{1
+\left(\frac{2t y_n}{\gap+t r_n}\right)^2}\right)^2}}~~,~~ t\in [0,1],
\]
and note that $g(1)$ equals $\text{sign}(\hat{v}^1_1)\cdot\hat{v}^1_2$ as defined 
above. By a Taylor expansion, we have
\[
\text{sign}(\hat{v}^1_1)\cdot\hat{v}^1_2 = g(1) = 
g(0)+g'(0)+\frac{1}{2}g''(0)+\frac{s^3}{6}g'''(s)
\]
for some $s\in [0,1]$. A tedious calculation of $g$'s derivatives\footnote{Using MATLAB's symbolic math toolbox together with some straightforward manual calculations} reveals that this implies
\begin{eqnarray}\label{eq:signFixLB2:3}
\text{sign}(\hat{v}^1_1)\cdot\hat{v}^1_2 ~=~ 
\frac{y_n}{\gap}-\frac{r_ny_n}{\gap^2}\pm 
\Ocal\left(\frac{|y_n|^3+|r_n|^3}{\gap^3}\right),
\end{eqnarray}
assuming $\max\{|y_n|,|r_n|\}\leq c\gap$ for some constant $c$ (hence fixing $c$ we used in our earlier assumptions on $r_n,z_n$). 
To simplify notation, let $q_n = \text{sign}(\hat{v}^1_1)\cdot\hat{v}^1_2$, let $b_n = 
\frac{y_n}{\gap}-\frac{r_ny_n}{\gap^2}\pm 
\Ocal\left(\frac{|y_n|^3+|r_n|^3}{\gap^3}\right)$ be the expression on the right-hand side of the equation above, and let $A$ be the event 
that $\max\{|y_n|,|r_n|\}\leq c\gap$ indeed holds. Also, note that with 
probability $1$, $|q_n|\leq 1$ and $|b_n| = \Ocal(1/\gap^3)$. Thus, by Eq. \eqref{eq:signFixLB2:3}, we have that $\E[q_n|A]=\E[b_n|A]$, and therefore
\begin{align*}
\E[q_n] &= 
\Pr(\neg A)\cdot\E[q_n|
\neg 
A]+\Pr(A)\cdot\E[q_n| A]\\
&= \Pr(\neg A)\cdot\E[q_n|
\neg 
A]+\Pr(A)\cdot \E[b_n|A] \\
&= \Pr(\neg A)\cdot\E[q_n|
\neg 
A]+\E[b_n]-\Pr(\neg A)\cdot \E[b_n|\neg A] \\
&= \E[b_n]\pm \Ocal\left(\Pr(\neg A)/\gap^3\right).
\end{align*}
Plugging back the definitions of $q_n,b_n,A$, we get that
\[
\E\left[\text{sign}(\hat{v}^1_1)\cdot\hat{v}^1_2\right] ~=~ 
\E\left[\frac{y_n}{\gap}-\frac{r_ny_n}{\gap^2}\pm 
\Ocal\left(\frac{|y_n|^3+|r_n|^3}{\gap^3}\right)\right]\pm 
\Ocal\left(\frac{1}{\gap^3}\Pr(\max\{|y_n|,|r_n|\}> c\gap)\right).
\]
Recalling that $y_n = \sqrt{1+\gap}\cdot\frac{1}{n}\sum_{i=1}^{n}\xi_i$ and 
$r_n = 1-z_n = 1-\frac{1}{n}\sum_{i=1}^{n}\xi_i^2$, where $\xi_i$ are i.i.d. 
copies of a zero-mean, bounded random variable satisfying $\E[\xi^3]=1/\sqrt{2}$, and using Hoeffding's inequality, it 
is easily verified that the above equals
\[
0 + \sqrt{1+\gap}\frac{1}{\sqrt{2}\gap^2 n}\pm\Ocal\left(\frac{1}{(\gap^2 
n)^{3/2}}\right)\pm\Ocal\left(\frac{1}{\gap^3}\exp(-\Omega(n\gap^2))\right),
\]
which is $\Omega\left(\frac{1}{\gap^2 n}\right)$ assuming $n$ is sufficiently larger than $1/\gap^2$. As a result, we get that $\left(\E\left[\text{sign}(\hat{v}^1_1)\cdot\hat{v}^1_2\right]\right)^2 = \Omega\left(\frac{1}{\gap^4 n^2}\right)$ as required.
\end{proof}

\section{Proofs Omitted from Section \ref{sec:multiRoundAlgs}}

\subsection{Proof of Lemma \ref{lem:approxERM4PCA}}
\begin{proof}
Let $\bullet$ denote the standard inner product for matrices, i.e., $\A\bullet\B = \trace(\A\B^{\top})$.
It holds that
\begin{eqnarray*}
(\w^{\top}\v_1)^2 &=& \w\w^{\top}\bullet\v_1\v_1^{\top} \geq \hat{\v}_1\hat{\v}_1^{\top}\bullet\v_1\v_1^{\top} - \Vert{\w\w^{\top}-\hat{\v}_1\hat{\v}_1^{\top}}\Vert_F\cdot\Vert{\v_1\v_1^{\top}}\Vert \\
&= & (\w^{\top}\v_1)^2- \sqrt{2(1-1(\w^{\top}\hat{\v}_1)^2)} \geq (\w^{\top}\v_1)^2 - \sqrt{2\epsilon}.
\end{eqnarray*}
\end{proof}

\subsection{Proof of Lemma \ref{lem:precondLS}}
\begin{proof}
Observe that $\C = \M + (\hat{\X} - \hat{\X}_1) + \mu\I$. Thus, by our assumption on $\mu$ it follows that
\begin{eqnarray}\label{eq:lem:precondLS:1}
\M + 2\mu\I \succeq  \C \succeq \M. 
\end{eqnarray}

Since $\tilde{F}_{\lambda,\w}(\y)$ is twice differentiable, in order to bound its smoothness and strong-convexity parameters, it suffices to upper bound the largest eigenvalue and lower bound the smallest eigenvalue of its Hessian, respectively.

The Hessian of $\tilde{F}_{\lambda,\w}(\y)$ is given by $\nabla^2\tilde{F}_{\lambda,\w}(\y) = \C^{-1/2}\M\C^{-1/2}$.


From Eq. \eqref{eq:lem:precondLS:1} it follows that we can write $\M = \C - \Delta$ where $\Delta \succeq 0$.

Thus we have that
\begin{eqnarray}\label{eq:lem:precondLS:3}
\lambda_1(\C^{-1/2}\M\C^{-1/2}) = \lambda_1(\C^{-1/2}(\C-\Delta)\C^{-1/2}) \leq  \lambda_1(\I) = 1,
\end{eqnarray}
where the inequality follows since $\C^{-1/2}\Delta\C^{-1/2}$ is positive semidefinite.

Since $\M,\C$ are invertible and positive definite, Eq. \eqref{eq:lem:precondLS:1} implies that
\begin{eqnarray}\label{eq:lem:precondLS:4}
\M^{-1} \succeq \C^{-1} \succeq (\M+2\mu\I)^{-1} .
\end{eqnarray}

Thus we have that
\begin{eqnarray}\label{eq:lem:precondLS:5}
\lambda_d(\C^{-1/2}\M\C^{-1/2}) &=& \lambda_d(\M^{1/2}\C^{-1/2}\C^{-1/2}\M\C^{-1/2}\C^{1/2}\M^{-1/2})
= \lambda_d(\M^{1/2}\C^{-1}\M^{1/2}) \nonumber\\
&\geq &  \lambda_d(\M^{1/2}(\M+2\mu\I)^{-1}\M^{1/2})
= \min_{i\in[d]}\{\frac{\lambda_i(\M)}{\lambda_i(\M)+2\mu}\} \nonumber \\
&=& \frac{\lambda_d(\M)}{\lambda_d(\M)+2\mu} = \frac{\lambda-\hat{\lambda}_1}{(\lambda-\hat{\lambda}_1) + 2\mu},
\end{eqnarray}
where the first equality follows from matrix similarity and the fact that $\M,\C$ are invertible, and the first inequality follows from Eq. \eqref{eq:lem:precondLS:4}.


To prove the second part of the lemma we observe that
\begin{eqnarray*}
\Vert{\tilde{\z} - \M^{-1}w}\Vert &=& \Vert{\C^{-1/2}\tilde{\y} - \C^{-1/2}\C^{1/2}\M^{-1}w}\Vert \leq \Vert{\C^{-1/2}}\Vert\cdot\Vert{\tilde{\y} - \C^{1/2}\M^{-1}\w}\Vert \\
&\leq &  \frac{1}{\sqrt{\lambda-\lambda_1(\hat{\X})}}\Vert{\tilde{\y} - \C^{1/2}\M^{-1}\w}\Vert,
\end{eqnarray*}
where the second inequality follows from Eq. \eqref{eq:lem:precondLS:4}.

Finally, the last part of the lemma follows from a direct application of Theorem \ref{thm:matHoff} to upper bound $\Vert{\X-\hat{\X}_1}\Vert$.
\end{proof}

\subsection{Proof of Lemma \ref{lem:distPrecondLS}}

\begin{proof}
Let $\z^* := (\lambda\I-\hat{\X})^{-1}\w, \y^* := \C^{1/2}(\lambda\I-\hat{\X})^{-1}\w$, and recall that $\z^*$ and $\y^*$ are the global minimizers of $F_{\lambda,\w}(\z)$ and $\tilde{F}_{\lambda,\w}(\y)$, respectively.
Using the results of Lemma \ref{lem:precondLS} we have that
\begin{eqnarray*}
\Vert{\tilde{\z} - \z^*}\Vert \leq (\lambda-\hat{\lambda}_1)^{-1/2}\Vert{\tilde{\y}-\y^*}\Vert \leq (\lambda-\hat{\lambda}_1)^{-1/2}\sqrt{2\left({1+ \frac{2\mu}{\lambda-\hat{\lambda}_1}}\right)\epsilon'},
\end{eqnarray*}
where the second inequality follows from the strong-convexity of $\tilde{F}_{\lambda,\w}(\y)$.
Thus, it suffices to set $\epsilon'$ as stated in the lemma in order to obtain the approximation guarantee for $\tilde{\z}$.

To upper-bound the total number of communication rounds required to obtain $\tilde{\y}$ with the guarantee prescribed in the lemma, we note that both the conjugate gradient method and Nesterov's accelerated gradient method require 
\begin{eqnarray}\label{eq:lem:distLS:1}
O\left({\sqrt{\frac{\beta}{\alpha}}\ln\left({\Vert{\y^*}\Vert/\epsilon'}\right)}\right)
\end{eqnarray}
calls to the first-order oracle of $\tilde{F}_{\lambda,\w}(\y)$ to  obtain 
$\tilde{\y}$ satisfying $\tilde{F}_{\lambda,\w}(\tilde{\y}) - 
\min_{\y\in\reals^d}\tilde{F}_{\lambda,\w}(\y) \leq \epsilon'$, where $\alpha$ 
and $\beta$ are the strong-convexity and smoothness parameters of 
$\tilde{F}_{\lambda,\w}$, respectively, and assuming w.l.o.g. that the initial 
iterate is $\y_0 = \vec{\textbf{0}}$. Thus, by our construction of a 
distributed first-order oracle given in Algorithm \ref{alg:distFOO}, we have 
that the total number of communication rounds is upper bounded by 
\eqref{eq:lem:distLS:1}.
The lemma now follows from noticing that by Lemma \ref{lem:precondLS} we have that $\beta/\alpha = 1 + \frac{2\mu}{\lambda-\hat{\lambda}_1}$ and that 
\begin{eqnarray*}
\Vert{\y^*}\Vert = \Vert{\C^{1/2}(\lambda\I-\hat{\X}_1)\w}\Vert \leq \lambda_1(\C^{1/2})(\lambda-\hat{\lambda}_1)^{-1}\Vert{\w}\Vert = O\left({\Vert{\w}\Vert/(\lambda-\hat{\lambda}_1)}\right).
\end{eqnarray*}
\end{proof}

\subsection{Proof of Theorem \ref{thm:shiftNinvert:main}}
\begin{proof}
Under our assumption that $mn = \Omega(\delta^{-2}\ln(d/p))$, the following three events all hold with probability at least $1-p$ (each of which holds w.p. at least $1-p/3$):
\begin{enumerate}
\item
the output $\w_f$ satisfies $(\w_f^{\top}\hat{\v}_1)^2 \geq 1-\epsilon$ (holds w.p. $1-p/3$ by applying Lemma \ref{lem:convexEV} with our choice of parameters)
\item 
$\hat{\delta} = \Theta(\delta)$ (by applying Theorem \ref{thm:matHoff})
\item 
$\Vert{\hat{\X} - \hat{\X}_1}\Vert \leq \mu$, where $\mu$ is as prescribed in the Theorem (by applying Theorem \ref{thm:matHoff})
\end{enumerate}
The approximation guarantee of $\w_f$ follows directly from Lemma \ref{lem:convexEV}. It thus remains to upper-bound the number of matrix-vector products.
Thus, combining Lemmas \ref{lem:convexEV} and \ref{lem:distPrecondLS} we have that when using either the conjugate gradient method or Nesterov's accelerated method to approximately solve the linear systems in Algorithm \ref{alg:convexEV}, as prescribed in Lemma \ref{lem:distPrecondLS}, the total number of distributed matrix-vector products with $\hat{\X}$ is:
\begin{eqnarray*}
&&O\left({\ln\left({\frac{d}{p\epsilon}}\right)\cdot\left({\sqrt{1 + \frac{2\mu}{\delta}}\left({\ln\delta^{-1}\ln\left({\frac{d}{p\epsilon}}\right) + \ln\left({\frac{\left({1 + 2\mu/\delta}\right)}{\delta\tilde{\epsilon}}}\right)}\right)}\right)}\right) = \\
&& O\left({\sqrt{1 + \frac{2\mu}{\delta}}\left({\ln\delta^{-1}\ln^2\left({\frac{d}{p\epsilon}}\right) + \ln\left({\frac{d}{p\epsilon}}\right)\left({\ln\left({\frac{\left({1 + 2\mu/\delta}\right)}{\delta}}\right)+\ln\left({\frac{1}{\tilde{\epsilon}}}\right)}\right)}\right)}\right) = \\
&& O\left({\sqrt{1 + \frac{2\mu}{\delta}}\left({\ln\delta^{-1}\ln^2\left({\frac{d}{p\epsilon}}\right) + \ln\left({\frac{d}{p\epsilon}}\right)\ln\left({\frac{\left({1 + 2\mu/\delta}\right)}{\delta}}\right)+\ln^2\left({\frac{d}{p\epsilon}}\right)\ln\left({\frac{1}{\delta}}\right)}\right)}\right)
,
\end{eqnarray*}
where the first term in the $O(\cdot)$ in the first row accounts for the total number of instances of $F_{\lambda,\w}(\z)$ needs to be solved, given by the bound in Lemma \ref{lem:convexEV}, and the second term in the first row accounts for the communication-complexity of solving each such instance according to Lemma \ref{lem:distPrecondLS}. Additionally, we have used Lemma \ref{lem:convexEV} to lower bound $\lambda-\hat{\lambda}_1 = \Omega(\hat{\delta})$, and $\tilde{\epsilon}(\epsilon)$ is as prescribed in Algorithm \ref{alg:convexEV}. Finally, we have upper-bounded $\ln(\Vert{\w}\Vert)$, in all instances of $F_{\lambda,\w}(\z)$ solved throughout the run of the algorithm, by noticing that in all of them it holds that 
\begin{eqnarray*}
\ln(\Vert{\w}\Vert) = O\left({\ln\left({\lambda_{(s)}-\hat{\lambda}_1)^{-\max\{m_1,m_2\}}}\right)}\right) = O\left({\ln\delta^{-1}\ln\left({\frac{d}{p\epsilon}}\right)}\right),
\end{eqnarray*}
where $m_1,m_2$ are as prescribed in Algorithm \ref{alg:convexEV}, and we have used Lemma \ref{lem:convexEV} again to lower bound $\lambda_{(s)}-\hat{\lambda}_1 = \Omega(\delta)$.

Finally, using Lemma \ref{lem:precondLS}, we can set $\mu =  \frac{4\sqrt{\ln(3d/p)}}{\sqrt{n}}$. Thus, the overall number of communication rounds is upper-bound by
\begin{eqnarray*}
&& O\left({\sqrt{\frac{\sqrt{\ln(d/p)}}{\delta\sqrt{n}}}\left({\ln\left({\frac{d}{p\epsilon^2}}\right)\ln\left({\frac{\sqrt{\ln(d/p)}}{\delta^2\sqrt{n}}}\right)+\ln^2\left({\frac{d}{p\epsilon^2}}\right)\ln\left({\frac{1}{\delta}}\right)}\right)}\right).
\end{eqnarray*}
\end{proof}

\section{Proof of the Davis-Kahan sin$\theta$ Theorem}

We prove Theorem \ref{thm:DK} in greater generality. In particular, Theorem \ref{thm:DK} follows from setting $k=1$ in the next theorem.

\begin{theorem}[Davis-Kahan sin$\theta$ theorem]
Let $\X,\Y$ be symmetric real $d\times d$ matrices and fix $k\in[d]$. Let $\V_{\X}$ and $\V_{\Y}$ denote $d\times k$ matrix whose columns are the top $k$ eigenvectors of $\X$ and the matrix whose columns are the top $k$ eigenvectors of $\Y$, respectively. Also, suppose that $\delta_k(\X) := \lambda_k(\X) - \lambda_{k+1}(\X) > 0$. Then it holds that
\begin{eqnarray*}
\Vert{\V_{\X}\V_{\X}^{\top} - \V_{\Y}\V_{\Y}^{\top}}\Vert_F \leq 2\frac{\Vert{\X-\Y}\Vert}{\delta_k(\X)}.
\end{eqnarray*}
\end{theorem}

\begin{proof}
Throughout the proof we denote the projection matrices:
\begin{eqnarray*}
\P_{\X} := \V_{\X}\V_{\X}^{\top},~~  \P_{\X}^{\perp} := \I - \V_{\X}\V_{\X}^{\top}, ~~
\P_{\Y} := \V_{\Y}\V_{\Y}^{\top}, ~~  \P_{\Y}^{\perp} := \I - \V_{\Y}\V_{\Y}^{\top},
\end{eqnarray*}
i.e., $\P_{\X}$ is the projection matrix onto the top $k$ eigenvectors of $\X$ and $\P_{\X}^{\perp}$ is the projection matrix onto the lower $d-k$ eigenvectors, and same goes for $\P_{\Y}, \P_{\Y}^{\perp}$.
We also let $\A\bullet\B$ denote the standard inner products between matrices $\A,\B$.

We can write $\P_{\Y}$ as
\begin{eqnarray}\label{eq:dk:1}
\P_{\Y} &=& \P_{\X}\P_{\Y}\P_{\X} + \P^{\perp}_{\X}\P_{\Y}\P_{\X} + \P_{\X}\P_{\Y}\P^{\perp}_{\X} + \P^{\perp}_{\X}\P_{\Y}\P^{\perp}_{\X}.
\end{eqnarray}


Observe that
\begin{eqnarray}\label{eq:dk:2}
\P_{\X}\P_{\Y}\P^{\perp}_{\X}\bullet\X = \trace\left({\P_{\X}\P_{\Y}\P^{\perp}_{\X}\X}\right) = \trace\left({\P_{\Y}\P^{\perp}_{\X}\X\P_{\X}}\right) = 0,
\end{eqnarray}
where the second equality follows from the cyclic property of the trace, and the last equality follows since $\P^{\perp}_{\X}\X\P_{\X} = \textbf{0}_{d\times d}$. 
Using Eq. \eqref{eq:dk:1} and \eqref{eq:dk:2} we have that
\begin{eqnarray}\label{eq:dk:4}
\P_{\Y} \bullet \X &=& \P_{\X}\P_{\Y}\P_{\X} \bullet \X + \P^{\perp}_{\X}\P_{\Y}\P^{\perp}_{\X}\bullet\X = \trace\left({\P_{\X}\P_{\Y}\P_{\X}\X}\right) + \trace\left({\P^{\perp}_{\X}\P_{\Y}\P^{\perp}_{\X}\X}\right) \nonumber \\
&=& \trace\left({\P_{\Y}\P_{\X}\X}\right) + \trace\left({\P^{\perp}_{\X}\P_{\Y}\P^{\perp}_{\X}\P^{\perp}_{\X}\X}\right) 
\leq  \trace\left({\P_{\Y}\P_{\X}\X}\right) + \trace\left({\P^{\perp}_{\X}\P_{\Y}\P^{\perp}_{\X}}\right)\cdot\lambda_1(\P^{\perp}_{\X}\X) \nonumber \\
&=& \trace\left({\P_{\Y}\P_{\X}\X}\right) + \lambda_{k+1}(\X)\cdot\trace\left({\P^{\perp}_{\X}\P_{\Y}}\right),
\end{eqnarray}
where the inequality follows since for any two positive semidefinite matrices $\A,\B$ it holds that $\trace(\A\B) \leq \trace(\A)\cdot\lambda_1(\B)$ and the fact that $\P^{\perp}_{\X}\X$ is positive semidefinite. The last equality follows since $\lambda_1(\P^{\perp}_{\X}\X) = \lambda_{k+1}(\X)$.
It further holds that
\begin{eqnarray}\label{eq:dk:5}
\P_{\Y} \bullet \Y &\geq & \P_{\X}\bullet\Y = \trace(\P_{\X}\X) + \P_{\X} \bullet (\Y-\X) . 
\end{eqnarray}

Subtracting Eq. \eqref{eq:dk:5} from Eq. \eqref{eq:dk:4} we have that
\begin{eqnarray*}
\trace\left({\P_{\Y}\P_{\X}\X}\right) + \lambda_{k+1}(\X)\cdot\trace\left({\P^{\perp}_{\X}\P_{\Y}}\right) - \trace(\P_{\X}\X) - \P_{\X} \bullet (\Y-\X) 
\geq \P_{\Y} \bullet \X - \P_{\Y} \bullet \Y .
\end{eqnarray*}

Rearranging we have that
\begin{eqnarray}\label{eq:dk:6}
\trace\left({(\I - \P_{\Y})\P_{\X}\X}\right) - \lambda_{k+1}(\X)\cdot\trace\left({\P^{\perp}_{\X}\P_{\Y}}\right) &\leq & (\Y-\X)\bullet(\P_{\Y} - \P_{\X})  \nonumber \\
&\leq &\Vert{\X-\Y}\Vert\cdot\Vert{\P_{\X}- \P_{\Y}}\Vert_F.
\end{eqnarray}

It holds that
\begin{eqnarray}\label{eq:dk:7}
\trace\left({(\I - \P_{\Y})\P_{\X}\X}\right)&=& \trace\left({\P_{\X}(\I - \P_{\Y})\P_{\X}\P_{\X}\X}\right) \nonumber \\
&\geq & \trace\left({\P_{\X}(\I - \P_{\Y})\P_{\X}}\right)\cdot\lambda_k(\P_{\X}\X) \nonumber \\
&= & \trace\left({\P_{\X} - \P_{\Y}\P_{\X})}\right)\cdot\lambda_k(\X) \nonumber \\
&=& \left({k -\P_{\X}\bullet\P_{\Y}}\right)\cdot\lambda_k(\X)  \nonumber \\
&=& \frac{\lambda_k(\X)}{2}\Vert{\P_{\X}-\P_{\Y}}\Vert_F^2.
\end{eqnarray}

Furthermore, it holds that
\begin{eqnarray}\label{eq:dk:8}
\trace\left({\P^{\perp}_{\X}\P_{\Y}}\right) = \trace\left({(\I-\P_{\X})\P_{\Y}}\right) = k - \P_{\X}\bullet{\P_{\Y}} = \frac{1}{2}\Vert{\P_{\X}-\P_{\Y}}\Vert_F^2.
\end{eqnarray}

Plugging Eq. \eqref{eq:dk:7} and \eqref{eq:dk:8} into Eq. \eqref{eq:dk:6}, we have that
\begin{eqnarray}\label{eq:dk:9}
\frac{1}{2}\Vert{\P_{\X}-\P_{\Y}}\Vert_F^2\cdot(\lambda_k(\X) - \lambda_{k+1}(\X)) \leq \Vert{\X-\Y}\Vert\cdot\Vert{\P_{\X}-\P_{\Y}}\Vert_F,
\end{eqnarray}
which completes the proof.
\end{proof}

\end{document}